\documentclass[a4paper,fleqn]{cas-dc}

\usepackage[authoryear,longnamesfirst]{natbib}

\usepackage{physics}
\usepackage{graphicx}
\usepackage{bbm}
\usepackage{hyperref}

\newcommand{\Var}[1]{\operatorname{Var}\left[#1\right]}
\newcommand{\E}[1]{\operatorname{\mathbb{E}}\left[#1\right]}
\newcommand{\Proba}[1]{\operatorname{\mathbb{P}}\left[#1\right]}
\newcommand{\Cov}[1]{\operatorname{Cov}\left[#1\right]}
\newcommand{\emptyline}{\raisebox{0.8ex}{\rule{24pt}{0.4pt}}}
\newcommand{\iu}{{i\mkern1mu}}

\def\tsc#1{\csdef{#1}{\textsc{\lowercase{#1}}\xspace}}
\tsc{WGM}
\tsc{QE}
\tsc{EP}
\tsc{PMS}
\tsc{BEC}
\tsc{DE}

\newtheorem{theorem}{Theorem}
\newproof{proof}{Proof}

\begin{document}
\let\WriteBookmarks\relax
\def\floatpagepagefraction{1}
\def\textpagefraction{.001}
\shorttitle{1-Lipschitz Network Initialization}
\shortauthors{M.F.R. Juston et~al.}

\title [mode = title]{1-Lipschitz Network Initialization for Certifiably Robust Classification Applications: A Decay Problem}                      



\author[1]{Marius {F.~R. Juston}}[type=editor,
                        auid=000,bioid=1,
                        role=Researcher,
                        orcid=0009-0005-7216-4180]
\cormark[1]
\ead{mjuston2@illinois.edu}

\credit{Writing – review \& editing, Writing – original draft, Methodology, Formal analysis, Conceptualization, Software}

\affiliation[1]{organization={The Grainger College of Engineering, Industrial and Enterprise Systems Engineering Department},
                addressline={University of Illinois Urbana-Champaign}, 
                city={Urbana},
                postcode={695013}, 
                city={Urbana},
                state={Illinois},
                country={USA}}

\author[1]{Ramavarapu {S. Sreenivas}}[type=editor,
                        auid=000,bioid=1,
                        orcid=0000-0002-8242-0839]
\ead{rsree@illinois.edu}

\credit{Writing – review \& editing}

\author[1]{William {R. Norris}}[%
    orcid=0000-0002-4940-4458
   ]
\fnmark[1]
\ead{wrnorris@illinois.edu}

\credit{Writing - Review}

\affiliation[2]{organization={Construction Engineering Research Laboratory},
                addressline={U.S. Army Corps of Engineers Engineering Research and Development Center}, 
                postcode={61822}, 
                city={Urbana},
                state={Illinois},
                country={USA}}

\author[2]{Ahmet Soylemezoglu}
\ead{ahmet.soylemezoglu@erdc.dren.mil}

\credit{Writing - Review}

\author[2]{Dustin Nottage}[
  orcid=0000-0003-4950-099X
]
\ead{dustin.S.nottage@erdc.dren.mil}

\credit{Writing - Review}


\cortext[cor1]{Corresponding author}
\fntext[fn1]{This author has passed away.}

\nonumnote{The code is available for replication of the results at https://github.com/Marius-Juston/SLLLipschitzInitialization}

\begin{abstract}
This paper discusses the weight parametrization of two standard 1-Lipschitz network architectures, the Almost-Orthogonal-Layers (AOL) and the SDP-based Lipschitz Layers (SLL). It examines their impact on initialization for deep 1-Lipschitz feedforward networks, and discusses underlying issues surrounding this initialization. These networks are mainly used in certifiably robust classification applications to combat adversarial attacks by limiting the impact of perturbations on the classification output. Exact and upper bounds for the parameterized weight variance were calculated assuming a standard Normal distribution initialization; additionally, an upper bound was computed assuming a Generalized Normal Distribution, generalizing the proof for Uniform, Laplace, and Normal distribution weight initializations. It is demonstrated that the weight variance holds no bearing on the output variance distribution and that only the dimension of the weight matrices matters. Additionally, this paper demonstrates that the weight initialization always causes deep 1-Lipschitz networks to decay to zero.  
\end{abstract}




\maketitle

\section{Introduction} \label{s_Intro}

The robustness of deep neural networks, primarily against adversarial attacks, has been a significant challenge in the field of modern applications of machine learning \citet{Nguyen2015, Szegedy2013, Biggio2013} by manipulating the input so that the model produces incorrect output. The problem of network robustness in deep networks stems mainly from the fact that large network weight magnitudes have an exponential impact on the output, the deeper it goes. The significant weight magnitudes thus enable a small perturbation to the input to cause a drastic change in the classification output \citet{Costa2024}.   
\par
The design of the 1-Lipschitz neural network has provided a reliable solution to certifying the network to be robust, such that the decision output remains the same within a sphere of perturbation \citet{Tsuzuku2018}. For the design, multiple approaches have been proposed, ranging from utilizing Spectral Normalization (SN) \citet{Miyato2018, Roth2020}, Orthogonal Parametrization \citet{Trockman2021}, Convex Potential Layers (CPL) \citet{Meunier2022}, Almost-Orthogonal-Layers (AOL) \citet{Prach2022}, Sandwich layers \citet{Wang2023DirectNetworks} and the recent SDP-based Lipschitz Layers (SLL) \citet{Araujo2023}. All these previous techniques assume standard Kaiming initialization scheme for its parameterized layered weights, which might or might not be appropriate for these networks. This remains under-studied for these networks; paper aims to further the theoretical initialization dynamics understanding of AOL and SLL networks. 
\par
This paper explores the impact of the weight parameterization of 1-Lipschitz networks employing Almost-Orthogonal-Layers and SDP-based Lipschitz Layers on the initialization of deep neural networks. Exploring the challenges in applying certifiably robust neural networks, such as 1-Lipschitz neural networks, is crucial, especially as neural network attacks become more frequent and robust classification results become more important. As such, discussing ways to improve training for deeper neural network architectures is important to understand and address. This article aims to illuminate some of the issues underlying these weight-normalizing networks.

\begin{itemize}
    \item An extended derivation for the network layer variance while accounting for the bias term and its recursive definition using the ReLU activation function is provided.
    \item Given the structure for the Almost-Orthogonal-Layers and SDP-based Lipschitz Layers feedforward network structure and weight parameterization, an upper bound and exact network weight variance is derived assuming a normal distribution initialization
    \item A general upper bound based on the Generalized Normal Distribution for the parameterized network weight variance is derived.
    \item Based on the calculated weight variance, insights for the 1-Lipschitz network are discussed as to potential issues in this network's initialization. 
\end{itemize}

The initial work for the initialization analysis is inspired by the works of Kaiming \citet{He2015} and Xavier \citet{Glorot2010}, while the 1-Lipschitz network structure is derived from \citet{Araujo2023}.

Understanding the initialization characteristics of these networks is extremely important as it allows researchers to better understand how to generate deeper Lipschitz networks, which enable robust certification guarantees, thus the bounds derived in the paper can help with researchers to develop more novel initialization or parameterization methods to help mitigate the issues that arise from the current state of the art methods.

\section{Related Work}

The starting work from Xavier \citet{Glorot2010} was a pivotal moment for deep neural networks with the methodology to properly initialize deep neural networks such that they would converge, assuming hyperbolic tangent activation functions; however, their work posed simplifying assumptions on the activation functions which caused issues when transition to more modern activation functions such as the commonly used ReLU \citet{Nair2010}.
\par
The works by Kaiming expanded on this concept and developed a method to generate weight initializations for deep networks using the Parameterized ReLU family \citet{He2015}; this work demonstrated the ability to ensure that the network would converge and train properly, regardless of depth. Since then, all modern machine learning has used Kaiming initialization for its networks, and modifications to the initialization gain have been activation-specific to ensure the stability criteria derived by Kaiming remain satisfied, as in the work on SELUs \citet{Klambauer2017}. An issue with the works above is the assumption of a bias term initialized to zero. From these works the connection to the correct initialization choice to enabling gradient stability establishes a direct link to impact the training accuracy and thus the robustness of networks.
\par
In conjunction with the works for network initialization, \citet{Araujo2023} developed a unifying methodology to combine multiple existing 1-Lipschitz network structures into a unifying framework. This framework provides a guideline for creating a new, certifiably robust neural network. The authors achieve this by formulating feedforward networks as a nonlinear robust control Lur'e system \citet{lur1944theory} and enforcing conditions on the generalized residual network's weights via SDP constraints. From this work, they can demonstrate general conditions for enforcing a multilayered 1-Lipschitz network and combine previous works from Spectral Normalization (SN) \citet{Miyato2018, Roth2020}, Orthogonal Parameterization \citet{Trockman2021}, Convex Potential Layers (CPL) \citet{Meunier2022}, and Almost-Orthogonal-Layers (AOL) \citet{Prach2022} into a single constraint. From the framework, they generate an augmented version of the AOL with additional parameterization, called SDP-based Lipschitz Layers, which improves the network's generalizability. However, previous work on robust networks also uses standard Kaiming initialization for network weights. In addition, because 1-Lipschitz activation functions are required, the ReLU activation function is commonly used and will also be used in the proofs below. This article explores the impact of using such an initialization scheme on networks. While the authors of \citet{Araujo2023} use the residual network, which, due to the additional interdependence, will be explored in future work, as convolution layers can be represented as a similar feedforward structure, the proof for the feedforward network generalizes to convolution layers \citet{Chetlur2014}. 

\section{Feed Forward Variance With Bias}

This article starts with a definition similar to the Kaiming \citet{He2015} and Xavier \citet{Glorot2010} initialization schemes; however, unlike their implementations, which set the bias to zero, this assumption is not made. The bias term is assumed to be a normally distributed, IID variable, similar to the weight matrix. The activation function is assumed to be ReLU for this derivation, as it is used in SLL and AOL networks. The desired end goal was to find $\Var{y_l}$. The variable $y_l$ was defined as:
\begin{align}
\boldsymbol{y_l} &=  W_l T_l^{-\frac{1}{2}}\boldsymbol{x_l} + \boldsymbol{b_l} \nonumber \\
\boldsymbol{x_l} &= \sigma(\boldsymbol{y_{l - 1}}).
\end{align}
Where $\sigma(x) = \max(0, x)$, which was the ReLU activation function. The matrix $T_l$ is a positive definite diagonal matrix as defined by the SLL 1-Lipschitz function definition \citet{Araujo2023}:
\begin{align}
T_l &= \text{diag}\left(\sum_{j = 1}^n \abs{W_l^T W_l }_{ij} \frac{q_j}{q_i}\right), q_i > 0 .
\end{align}
In this article, the parameter $q_i$ is initialized to the constant $1$; when set to the unit vector, the SLL $T_l$ derivation also encapsulates the AOL parameterization \citet{Prach2022}. The vectors were defined such that $x_l \in \mathbb{R}^{n_l \times 1}$, $W_l \in \mathbb{R}^{d_l \times n_l}$, $b_l \in \mathbb{R}^{d_l \times 1}$ with the following assumptions:
\begin{itemize}
    \item The initialized elements in $W_l$  were mutually independent and shared the same distribution $\forall l, j$, $\Cov{W_l, W_j} = 0$ with $l \neq j$,  and that $\operatorname{Var}[W_1] = \cdots = \Var{W_l}$
    \item Likewise, the elements in $x_l$  were mutually independent and shared the same distribution. $\forall l, j$, $\Cov{x_l, x_j} = 0$ with $l \neq j$ , and that $\Var{x_1} = \cdots  = \Var{x_l}$
     \item Additionally, the elements in $b_l$  were mutually independent and shared the same distribution. $\forall l, j$, $\Cov{b_l, b_j} = 0$ with $l \neq j$ , and that $\Var{b_1} = \cdots = \Var{b_l}$
    \item The vectors $x_l$, $W_l$ and, $b_l$ were independent of each other, $\Cov{W_l, x_l} = \Cov{W_l, b_l} = \Cov{b_l, x_l} = 0$
\end{itemize}
Under these assumptions, it could be determined that:
\begin{align}
\Var{\boldsymbol{y_l} } &= \Var{W_l x_l + b_l}  \nonumber \\
&= \Var{\begin{bmatrix}
    \sum_{j = 1}^{n_l} w_{1, j} x_j + b_1\\
    \sum_{j = 1}^{n_l} w_{2, j} x_j + b_2\\
    \vdots \\
    \sum_{j = 1}^{n_l} w_{d_j, j} x_j + b_{d_j}
\end{bmatrix}}  \nonumber \\
&= \sum_{k = 1}^{d_l}\left(\sum_{j = 1}^{n_l} \Var{w_l x_l} + \Var{b_l} \right) \nonumber \\
d_l \times \Var{y_l} &=  d_l \times  \left( n_l \times \Var{w_l x_l} + \Var{b_l} \right)  \nonumber \\
\Var{y_l} &= n_l \times \Var{w_l x_l} + \Var{b_l} .
\end{align}
The results were the same as those of Kaiming and Xavier, except for the addition of a bias term. Given the independence between the terms, the layer's variance could be expanded as:
\begin{align}
\Var{y_l} &= n_l \times \Var{w_l x_l} + \Var{b_l}  \nonumber \\
 &= n_l \times \left( \underbrace{\E{w_l^2}}_{\Var{w_l}} \E{x_l^2} - \underbrace{\E{w_l}^2}_{=0}\E{x_l}^2  \right) + \Var{b_l}  \nonumber \\
 &= n_l \Var{w_l} \E{x_l^2} + \Var{b_l}. \label{eqn:recusiveNonfullDefinition}
\end{align}
The $\E{x_l}$ does not have zero mean because the previous layer is $x_l = \max(0, y_{l - 1})$ and thus does not have zero mean. As such, $\E{x_l^2} = \E{\max(0, y_{l - 1})^2}$ needed to be handled.

\subsection{ReLU Expected Value}

In Kaiming's work, the expected value of the ReLU was derived; however, the bias term was set to zero. In contrast, the following expected value derivation includes bias in its computation:

Given that $b_{l -1}, w_{l - 1}, x_{l - 1}$ were independent.
\begin{align}
    \E{y_{l - 1}} &= \E{w_{l - 1} x_{l - 1} + b_{l- 1}}  \nonumber \\
        &= \underbrace{\E{w_{l - 1}}}_{=0}\E{ x_{l - 1}} + \underbrace{\E{b_{l- 1}}}_{=0} = 0.
\end{align}
\begin{theorem}
    Given an ReLU activation function, $\sigma(\cdot)$ the variance of the linear layer $y_l = \sigma(w_{l - 1} y_{l - 1} + b_{l - 1})$, where $\E{w_{l - 1}} = \E{b_{l - 1}} = 0$ and $y_{l - 1}$ is an unknown random variable has the following output variance, $\Var{y_{l-1}} = \E{y_{l-1}^2}$ and mean $\E{y_l} = 0$.
\end{theorem}

\begin{proof}
Because $w_{l - 1}$ and $b_{l - 1}$ have zero mean and were distributed symmetrically around zero:
\begin{align}
\Proba{b_{l- 1} > 0} &= \frac{1}{2}, \\
    \Proba{y_{l - 1} > 0} &= \Proba{w_{l - 1} x_{l - 1} + b_{l- 1} > 0}  \nonumber\\
                          &= \Proba{w_{l - 1} x_{l - 1}  > - b_{l- 1}} \nonumber\\
   &= \Proba{w_{l - 1} x_{l - 1}  > \left(b_{l- 1} > 0 \text{ or } b_{l - 1} < 0\right)} .
\end{align}
Using conditional probability:
\begin{align}
    \Proba{y_{l - 1} > 0}  =& \Proba{w_{l - 1} x_{l - 1}  > -b_{l-1} \rvert b_{l-1} > 0} \Proba{b_{l-1} > 0} \nonumber \\ &+ \Proba{w_{l - 1} x_{l - 1}  > -b_{l-1} \rvert b_{l-1} < 0} \Proba{b_{l-1} < 0}  \nonumber \\
    =& \frac{1}{2} \left( \Proba{w_{l - 1} x_{l - 1}  > -b_{l-1} \rvert b_{l-1} > 0} \right. \nonumber \\ &+ \left. \Proba{w_{l - 1} x_{l - 1}  > -b_{l-1} \rvert b_{l-1} < 0}  \right).
\end{align}
Using the property of symmetry around zero :
\begin{align}
    \Proba{w_{l - 1} x_{l - 1}  > t} = \Proba{w_{l - 1} x_{l - 1}  < -t }, \forall t \in \mathbb{R} .
\end{align}
Given this:
\begin{small}
\begin{align}
\Proba{w_{l - 1} x_{l - 1}  > -b_{l-1} \rvert b_{l-1} < 0} &= \Proba{w_{l - 1} x_{l - 1}  > b_{l-1} \rvert b_{l-1} > 0}  \nonumber \\
                                                           &=  \Proba{w_{l - 1} x_{l - 1}  < -b_{l-1}  \rvert b_{l - 1} > 0}.
\end{align}
\end{small}
As such:
\begin{align}
      \Proba{y_{l - 1} > 0}  
      =& \frac{1}{2} \left( \Proba{w_{l - 1} x_{l - 1}  > -b_{l-1} \rvert b_{l-1} > 0} \right. \nonumber \\  &+    \left. \Proba{w_{l - 1} x_{l - 1}  < -b_{l-1}  \rvert b_{l - 1} > 0}\right)  \nonumber \\
      =& \frac{1}{2} (1) = \frac{1}{2}.
\end{align}
As concluded, $y_{l - 1}$ was indeed centered on zero and symmetric around the mean. The expectation of $x_l^2$ could now be computed:
\begin{align}
    \E{x_l^2} &= \E{\max(0, y_{l - 1})^2}  \nonumber \\
                &= \Proba{y_{l - 1} < 0}\E{0} + \Proba{y_{l - 1} > 0}\E{ y_{l - 1}^2}  \nonumber\\
                &= \frac{1}{2}\E{ y_{l - 1}^2} = \frac{1}{2}\Var{ y_{l - 1}}.
\end{align}
\end{proof}

Plugging this back into \ref{eqn:recusiveNonfullDefinition}, it was computed that:
\begin{align}
    \Var{y_l} &= n_l \Var{w_l} \E{x_l^2} + \Var{b_l}  \nonumber \\
    &= \frac{n_l}{2}  \Var{w_l}  \Var{ y_{l - 1}}  + \Var{b_l} .
\end{align}
A recursive equation between the actions at layer $l$ and the activations at layer $l- 1$ was evaluated. Starting from the first layer, $2$, the following product was formed:
\begin{align}
   \Var{y_L} =& \prod_{l = 2}^L \left(\frac{n_l}{2}  \Var{w_l}\right) \Var{y_1}  \nonumber \\ &+ \sum_{l = 2}^{L - 1}  \left( \prod_{d=1}^{L - l}  \left(\frac{n_{L -d + 1}}{2}  \Var{w_{L -d + 1}} \right)\Var{b_l} \right) \nonumber \\  &+ \Var{b_L}  . \label{eqn:layerVariance}
\end{align}
Thus, this was a similar implementation to the network variance derivation determined by Kaiming, but with the bias term included.

\section{Transformed Weight Variance} 

The next step was to better understand what $\Var{w_l}$ was, given the network structure of the $WT^{-\frac{1}{2}}$. To see how the weight matrix was transformed, with the vector $q_i = 1$ as previously stated, an example weight matrix $W \in \mathbb{R}^{4 \times 2}$ was looked at. Following this weight dimension, the following transformation was acquired:
\begin{align}
    W &= \begin{bmatrix}
         a & c & e & g \\
         b & d & f & h
    \end{bmatrix}^T.
\end{align}
Where the transformed matrix was thus:
\begin{align}
 \bar{W} = W T^{-\frac{1}{2}} = W \text{diag}\left(\sum_{j=1}^n\abs{W^T W}_{ij}\right)^{-\frac{1}{2}} =
\end{align}
\begin{small}
\begin{align}
    \begin{bmatrix}
 \frac{a}{\sqrt{\left| a^2+c^2+e^2+g^2\right| +| a b+c d+e f+g h| }} & \frac{b}{\sqrt{| a b+c d+e f+g h| +\left|
   b^2+d^2+f^2+h^2\right| }} \\
 \frac{c}{\sqrt{\left| a^2+c^2+e^2+g^2\right| +| a b+c d+e f+g h| }} & \frac{d}{\sqrt{| a b+c d+e f+g h| +\left|
   b^2+d^2+f^2+h^2\right| }} \\
 \frac{e}{\sqrt{\left| a^2+c^2+e^2+g^2\right| +| a b+c d+e f+g h| }} & \frac{f}{\sqrt{| a b+c d+e f+g h| +\left|
   b^2+d^2+f^2+h^2\right| }} \\
 \frac{g}{\sqrt{\left| a^2+c^2+e^2+g^2\right| +| a b+c d+e f+g h| }} & \frac{h}{\sqrt{| a b+c d+e f+g h| +\left|
   b^2+d^2+f^2+h^2\right| }}
  \end{bmatrix}. \label{eqn:complexExample}  \nonumber
\end{align}
\end{small}
\subsection{Upper bound}

Given the highly complex distribution generated from this output, the system's complexity was reduced by looking at the upper bound approximation of \ref{eqn:complexExample}. This could be done by simply removing the off-diagonal terms in the normalization denominator as such:
\begin{align}
     W T^{-\frac{1}{2}} \leq \begin{bmatrix}
 \frac{a}{\sqrt{\left| a^2+c^2+e^2+g^2\right| }} & \frac{b}{\sqrt{\left|
   b^2+d^2+f^2+h^2\right| }} \\
 \frac{c}{\sqrt{\left| a^2+c^2+e^2+g^2\right| }} & \frac{d}{\sqrt{\left|
   b^2+d^2+f^2+h^2\right| }} \\
 \frac{e}{\sqrt{\left| a^2+c^2+e^2+g^2\right|  }} & \frac{f}{\sqrt{\left|
   b^2+d^2+f^2+h^2\right| }} \\
 \frac{g}{\sqrt{\left| a^2+c^2+e^2+g^2\right| }} & \frac{h}{\sqrt{\left|
   b^2+d^2+f^2+h^2\right| }}
  \end{bmatrix}.
\end{align}
More generally, for this upper bound, each element was thus represented as:
\begin{align}
    \bar{w}_i &= \frac{w_i}{\sqrt{\sum_{j = 1}^{d_l} w_j^2}}.
\end{align}
Given that $\E{w_i} = 0$, it was expected that the normalized expected value of $\E{\bar{w}_i} = 0 $ as well. As such, the variance for this system was defined as $\Var{\bar{w}_i} = \E{\bar{w}_i^2}$.
\begin{align}
    \E{\bar{w}_i^2} &= \E{\frac{w_i^2}{\sum_{j = 1}^{d_l} w_j^2}} = \E{\frac{w_i^2}{w_i^2 + \sum_{j = 1, j \neq i}^{d_l} w_j^2}}.
\end{align}
Given that the distribution $w_i \sim N(0, \sigma^2)$ the distribution $w_i^2$ represented a scaled Chi-Squared distribution defined as $w_i^2 \sim \sigma^2 \chi^2(1)$ which will be presented as $X$. The additional independent term $\sum_{j = 1, j \neq i}^{d_l} w_j^2$ was thus represented as a Chi-Squared distribution of the form $\sum_{j = 1, j \neq i}^{d_l} w_j^2  \sim \sigma^2 \chi^2(d_l - 1)$, represented as $Y$. The distribution thus followed the form:
\begin{align}
    \bar{w}_i^2 = \frac{X}{X + Y}. \label{eqn:upperDistribution}
\end{align}
\begin{theorem}
If $X$ and $Y$ are independent, with $X \sim \Gamma(\alpha, \theta)$ and $Y \sim \Gamma(\beta, \theta)$ then \citep[Theorem 3.15.]{1982SamplesDistributions}:
\begin{align}
    \frac{X}{X + Y} \sim \boldsymbol{B}(\alpha, \beta),
\end{align}
where $\boldsymbol{B}(\alpha, \beta)$ represents a Beta distribution and $\Gamma(\alpha, \theta)$ represents a Gamma distribution.
\end{theorem}
Given that Chi-Squared distributions can be represented as gamma distributions, $\Gamma(\alpha, \beta)$, with distribution parameters defined as $ \sigma^2 \chi^2(n) \sim \Gamma(\frac{n}{2}, \frac{1}{2 \sigma^2})$, the resultant Beta distribution and its expected value is thus:
\begin{align}
    \bar{w}_i^2 &\sim \boldsymbol{B}\left(\frac{1}{2}, \frac{d_l - 1}{2}\right), \\
    \E{\bar{w}_i^2} &= \frac{\alpha}{\alpha + \beta} 
                    = \frac{\frac{1}{2}}{\frac{1}{2} + \frac{d_l - 1}{2}} 
                    = \frac{1}{d_l}. \label{eqn:upperVariance}
\end{align}
This demonstrated that the resultant distribution variance between the layers depended only on the dimension $d_l$ of the weight matrix $W$, and that the initial distribution variance, $\sigma^2$, had no impact. This made creating an initialization scheme complicated, as no matter the initial variance of the weight matrix $W$, the output distribution would not be affected—only the matrix's dimension mattered, which was predetermined.

\subsection{Exact bound} \label{sec:exactBound}

In addition to deriving the upper bound approximation of the output distribution, the exact distribution's variance was computed by considering the off-diagonal terms. Similarly to the generalized distribution, the generalized form of each element was examined. Represented by the following formula:
\begin{align}
    \hat{w}_i &= \frac{w_i}{\sqrt{w_i^2 + \sum_{j = 1, j \neq i}^{d_l} w_j^2 + \sum_{j = 1, j \neq i}^{n_l}\left|w_i w_a + \sum_{k = 1}^{d_l - 1}w_b w_c \right|}}.
\end{align}
The weights $w_a, w_b, w_c$ were random elements from $W$; since the actual indexing does not affect the bound, the exact indexing is ignored. Similarly to the upper bound $\E{\hat{w}_i} =0$ and thus $\Var{\hat{w}_i} = \E{\hat{w}_i^2}$:
\begin{align}
    \hat{w}_i^2 &= \frac{w_i^2}{w_i^2 + \sum_{j = 1, j \neq i}^{d_l} w_j^2 + \sum_{j = 1, j \neq i}^{n_l}\left|w_i w_a + \sum_{k = 1}^{d_l - 1}w_b w_c \right|}. \label{eqn:variation_w}
\end{align}
As with the upper bound, the distribution $w_i^2 = X \sim \sigma^2 \chi^2(1)$ and $ \sum_{j = 1, j \neq i}^{d_l} w_j^2 = Y \sim \sigma^2 \chi^2(d_l - 1)$ were present; however, the additional term $\sum_{j = 1, j \neq i}^{n_l}\left|w_i w_a + \sum_{k = 1}^{d_l - 1}w_b w_c \right|$ $= Z$ provided a challenge as to what kind of distribution it would be. The $Z$ variable thus needed to be analyzed.

The product of two IID Gaussian samples, $w_b w_c$, denoted as a Normal Product Distribution, \citet{Weisstein2003} has the following probability density function (PDF) distribution:
\begin{align}
    p(w) &= \frac{K_0\left(\frac{| w| }{\sigma ^2}\right)}{\pi  \sigma ^2}
\end{align}
Where $K_n(z)$ was the modified Bessel function of the second kind \citet{822801}, 
\begin{align}
    K_n(z) &= \sqrt{\frac{\pi}{2z}} \frac{e^{-z}}{(n - \frac{1}{2})!} \int_0^\infty e^{-t}t^{n - \frac{1}{2}}\left(1 - \frac{t}{2z}\right)^{n - \frac{1}{2}}  dt,
\end{align}
To then derive the generalized sum of the Normal Product Distribution $\sum^{d_l}_{k = 1} w_{k,b} w_{k, c}$, this involves taking the convolution of the continuous probability distributions $d_l$ times; however, due to the modified Bessel function inside the PDF, this makes it difficult. Instead, the Fourier transform of the PDF can be taken,
\begin{align}
    \mathcal{F}(p(w)) &= \frac{1}{\sqrt{2 \pi } \sigma ^2 \sqrt{\frac{1}{\sigma ^4}+t^2}},
\end{align}
and then the convolution can be represented as taking the transformed function to the $n$-th power and inverting the transformed Fourier function,
\begin{align}
   \mathcal{F}^{-1}\left(\mathcal{F}(p(w))^n\right) &= \frac{2^{\frac{1}{2}-\frac{n}{2}} \sigma ^{-n-1} | w| ^{\frac{n-1}{2}} K_{\frac{n-1}{2}}\left(\frac{| w| }{\sigma ^2}\right)}{\sqrt{\pi } \Gamma
   \left(\frac{n}{2}\right)},
\end{align}
where $\Gamma(z)$ represents the Euler gamma function \citet{alma99327871512205899},
\begin{align}
    \Gamma(z) = \int_{0}^\infty t^{z-1}e^{-t}dt.
\end{align}
The expected value needed to be computed only when $w \ge 0$. The function is symmetric $p_w(w) = p_w(-w)$ and thus centered around zero, resulting in $P(w > 0) = \frac{1}{2}$. The original PDF function only needed to normalize a single side to generate a valid PDF. This thus resulted in the following output distribution.
\begin{align}
    p_{\abs{w}}(w) 
     &= \frac{1}{P_w(w \ge 0)} p_w(w \ge 0) \label{eqn:absoluteTaking} \\
    &= 2 p_w(w \ge 0) \\
            &= \frac{2^{\frac{3}{2}-\frac{n}{2}} \sigma ^{-n-1} w^{\frac{n-1}{2}} K_{\frac{n-1}{2}}\left(\frac{w}{\sigma ^2}\right)}{\sqrt{\pi } \Gamma
   \left(\frac{n}{2}\right)}.
\end{align}
Given this, the expected value could be computed as:
\begin{align}
    \E{p_{\abs{w}}} &= \int_{0}^\infty w p_\abs{w}(w) dw = \frac{2 \sigma ^2 \Gamma \left(\frac{n+1}{2}\right)}{\sqrt{\pi }
   \Gamma \left(\frac{n}{2}\right)} .
\end{align}
Sadly, this distribution could not be presented as a Gamma or Beta distribution. The trick for the upper bound cannot be used, as the distribution of $p_{\abs{w}}(w)$ depends on $w_i$; however, this explicit dependence was removed to simplify the computation and approximation. The expected value of the system is thus:
\begin{small}
\begin{align}
    \E{\hat{w}_i^2} &= \E{\frac{w_i^2}{w_i^2 + \sum_{j = 1, j \neq i}^{d_l} w_j^2 + \sum_{j = 1, j \neq i}^{n_l}\left|w_i w_a + \sum_{k = 1}^{d_l - 1}w_b w_c \right|}}  \nonumber\\
    &= \frac{\E{w_i^2}}{\E{w_i^2 + \sum_{j = 1, j \neq i}^{d_l} w_j^2 + \sum_{j = 1, j \neq i}^{n_l}\left|w_i w_a + \sum_{k = 1}^{d_l - 1}w_b w_c \right|}}  \nonumber\\
    &= \frac{\E{w_i^2}}{\E{w_i^2} + \E{\sum_{j = 1, j \neq i}^{d_l} w_j^2} + \E{\sum_{j = 1, j \neq i}^{n_l}\left|\sum_{k = 1}^{d_l }w_b w_c \right|}}  \label{eqn:expectedValueVariance} 
\end{align}
Substituting the expectation of each of the components, we get that, 
\begin{align}
   \E{\hat{w}_i^2} &= \frac{\sigma^2}{\sigma^2 +\sigma(d_l - 1) + \sum_{j = 1, j \neq i}^{n_l}\E{\left|\sum_{k = 1}^{d_l }w_b w_c \right|}}   \nonumber\\
    &= \frac{\sigma^2}{\sigma^2 +\sigma^2 (d_l - 1) + (n_l -1) \frac{2 \sigma ^2 \Gamma \left(\frac{d_l+1}{2}\right)}{\sqrt{\pi }
   \Gamma \left(\frac{d_l}{2}\right)}}  \nonumber \\
    &= \frac{1}{d_l + (n_l -1) \frac{2 \Gamma \left(\frac{d_l+1}{2}\right)}{\sqrt{\pi } \Gamma \left(\frac{d_l}{2}\right)} }. \label{eqn:exactVariance}
\end{align}
\end{small}
To make it more computationally stable, the logarithm of the $\Gamma(\cdot)$ function is usually used, as the factorial can become exceedingly large. This could be replaced with:
\begin{align}
    \frac{ \Gamma \left(\frac{d_l+1}{2}\right)}{ \Gamma \left(\frac{d_l}{2}\right)} &= e^{\ln {\Gamma \left(\frac{d_l+1}{2}\right)} - \ln{ \Gamma \left(\frac{d_l}{2}\right)}}.
\end{align}
The following simulated transformed weight matrix was sampled for varying values of $n_l$, with $d_l = 10 n_l$, to demonstrate that the output distribution variance was valid. Each sample point was evaluated at least 900,000 times to ensure the validity of the results. As shown in Figure \ref{fig:VarianceTransformedSimulation}, the upper bound in \ref{eqn:upperVariance} does indeed properly bound the variance of the weights. The theoretical variance computed in \ref{eqn:exactVariance} also matches the sampled distribution as noted through the perfect overlap. 

\begin{figure}
    \centering
    \includegraphics[width=0.9\columnwidth]{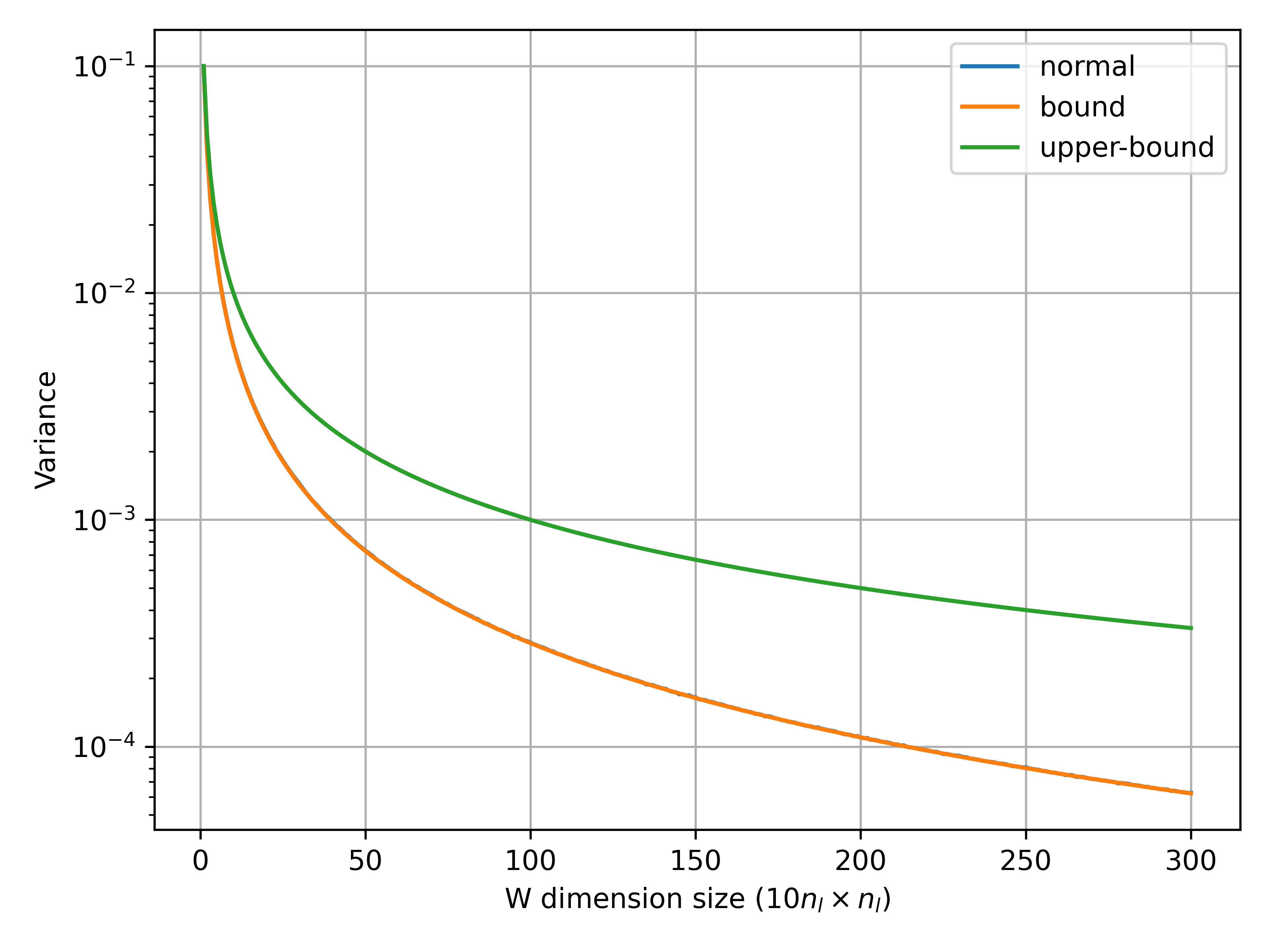}
    \caption{Transformed Weight Variance Simulation}
    \label{fig:VarianceTransformedSimulation}
\end{figure}

Thus, the derivation of $\Var{w_l}$ had been computed.

\subsection{Complete Forward Propagation Variance}

Given that the complete derivation of $\Var{w_l}$ had been computed, it could now be plugged back into the layer variance $\Var{y_L}$, in \ref{eqn:layerVariance}, for which the inner terms needed closer examination:
\begin{align}
    \frac{n_l}{2}\Var{w_l} &=  \frac{n_l}{2 d_l +  2 (n_l -1) \frac{2 \Gamma \left(\frac{d_l+1}{2}\right)}{\sqrt{\pi } \Gamma \left(\frac{d_l}{2}\right)}} .\label{eqn:decayEqn}
\end{align}
If the best case where $d_l = 1$ was assumed, this resulted in:
\begin{align}
    \frac{n_l}{2}\Var{w_l} &=  \frac{n_l}{2 +  2 (n_l -1) \frac{2 \Gamma \left(1\right)}{\sqrt{\pi } \Gamma \left(\frac{1}{2}\right)}},  \nonumber\\
     &=  \frac{n_l}{2 +  2 (n_l -1) \frac{2}{\pi}}, \nonumber \\
      &=  \frac{n_l}{2n_l + 2 - \frac{4}{\pi}}.
\end{align}
This informed us that no matter what the dimensionality of $\Var{w_l}$ in terms of $d_l$ or $n_l$, the output variance would always be less than $1$; even in the best case, it would converge to be $\frac{1}{2}$. This implied that given a sufficiently large $L$:
\begin{align}
   \prod_{l = 2}^L \left(\frac{n_l}{2}  \Var{w_l}\right) \Var{y_1} \approx 0 .
\end{align}
The bias term represented a converging geometric series given that the ratio term $r = \frac{n_l}{2}\Var{w_l} < 1$ (in this case, it was assumed that all layers had the same $d_l$ and $n_l$ to simplify the equation):
\begin{align}
    =& \sum_{l = 2}^{L - 1}  \left( \prod_{d=1}^{L - l}  \left(\frac{n_{L -d + 1}}{2}  \Var{w_{L -d + 1}} \right)\Var{b_l} \right) + \Var{b_L},  \nonumber \\
     =& \sum_{l = 2}^{L - 1}  \left(\Var{b_l}  \prod_{d=1}^{L - l}  r_d  \right) + \Var{b_L}       = \frac{\Var{b_l} }{1 - r}, \nonumber\\
    =& \frac{\Var{b_l} }{1 - \frac{n_l}{2 d_l +  2 (n_l -1) \frac{2 \Gamma \left(\frac{d_l+1}{2}\right)}{\sqrt{\pi } \Gamma \left(\frac{d_l}{2}\right)}}}.
\end{align}
Given that the bias term had a convergent property on the output layer variance, it did not truly matter what the variance of $\Var{b_l}$ was, as it would not cause the system to diverge and have exploding or vanishing output layer variances. To ensure that the output distribution's variance was close to one, the bias was set to:
\begin{align}
   \Var{y_L} =  \Var{b_l} &= 1 - \frac{n_l}{2 d_l +  2 (n_l -1) \frac{2 \Gamma \left(\frac{d_l+1}{2}\right)}{\sqrt{\pi } \Gamma \left(\frac{d_l}{2}\right)}} .\label{eqn:BiasInitialization}
\end{align}
Alternatively, $\Var{b_l} = 1$ provided similar results, as the variance function quickly decays to negligible values near zero. 
\par
This decay associated with the weight parameterization was demonstrated in Figure \ref{fig:HistVarianceActivation}, where a feedforward network with square weight matrices of dimension $8192 \times 8192$, using ReLU activations and weights initialized according to the standard Kaiming scheme with the bias term set to zero, was used. Figure \ref{fig:VarianceActivation} showed that the output distribution decayed quickly from the initial Gaussian input distribution with zero mean and variance of one. 

\begin{figure}
    \centering
    \includegraphics[width=0.9\columnwidth]{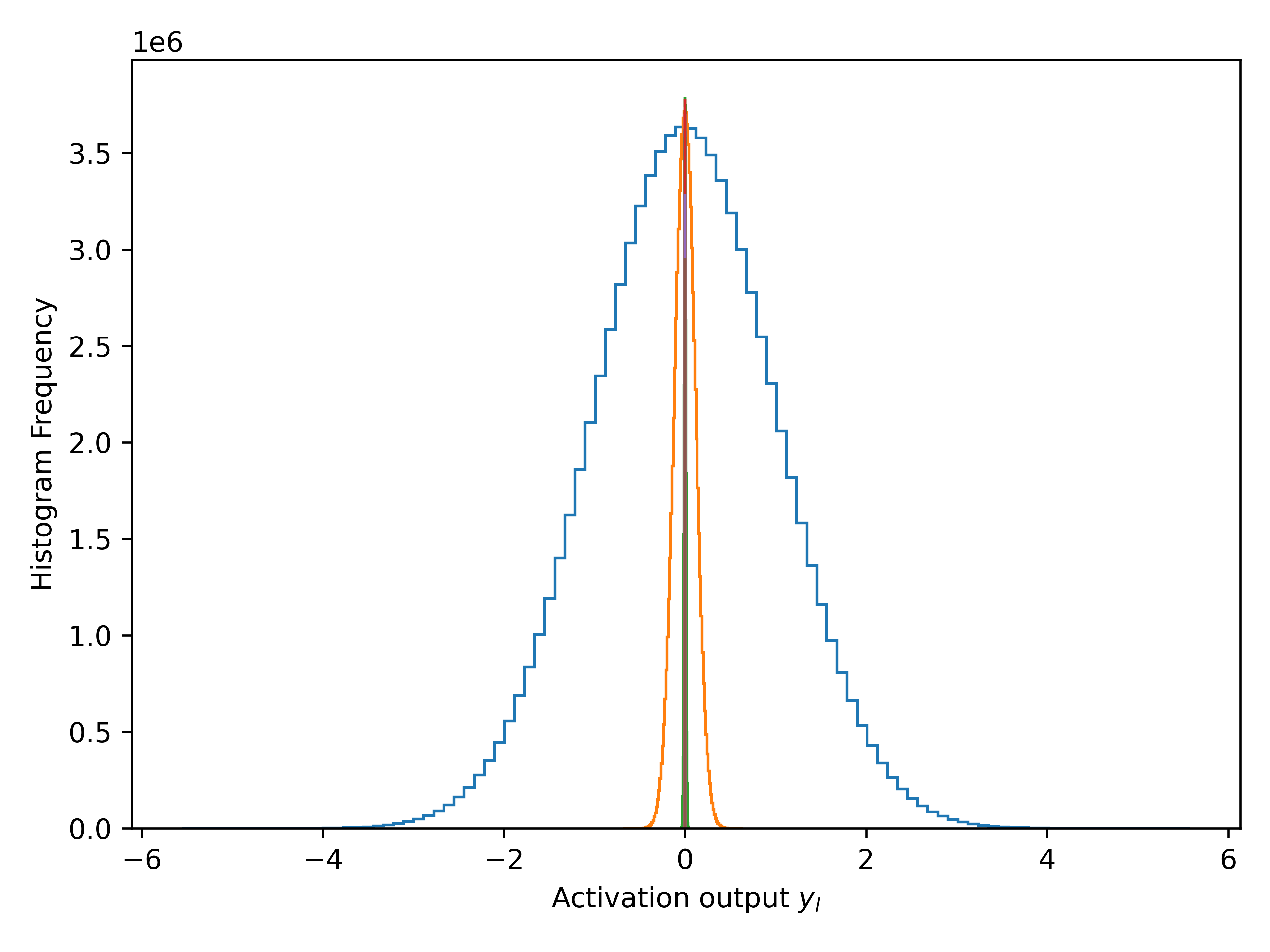}
    \caption{Forward Layer Activation Output }
    \label{fig:HistVarianceActivation}
\end{figure}

\begin{figure}
    \centering
    \includegraphics[width=0.9\columnwidth]{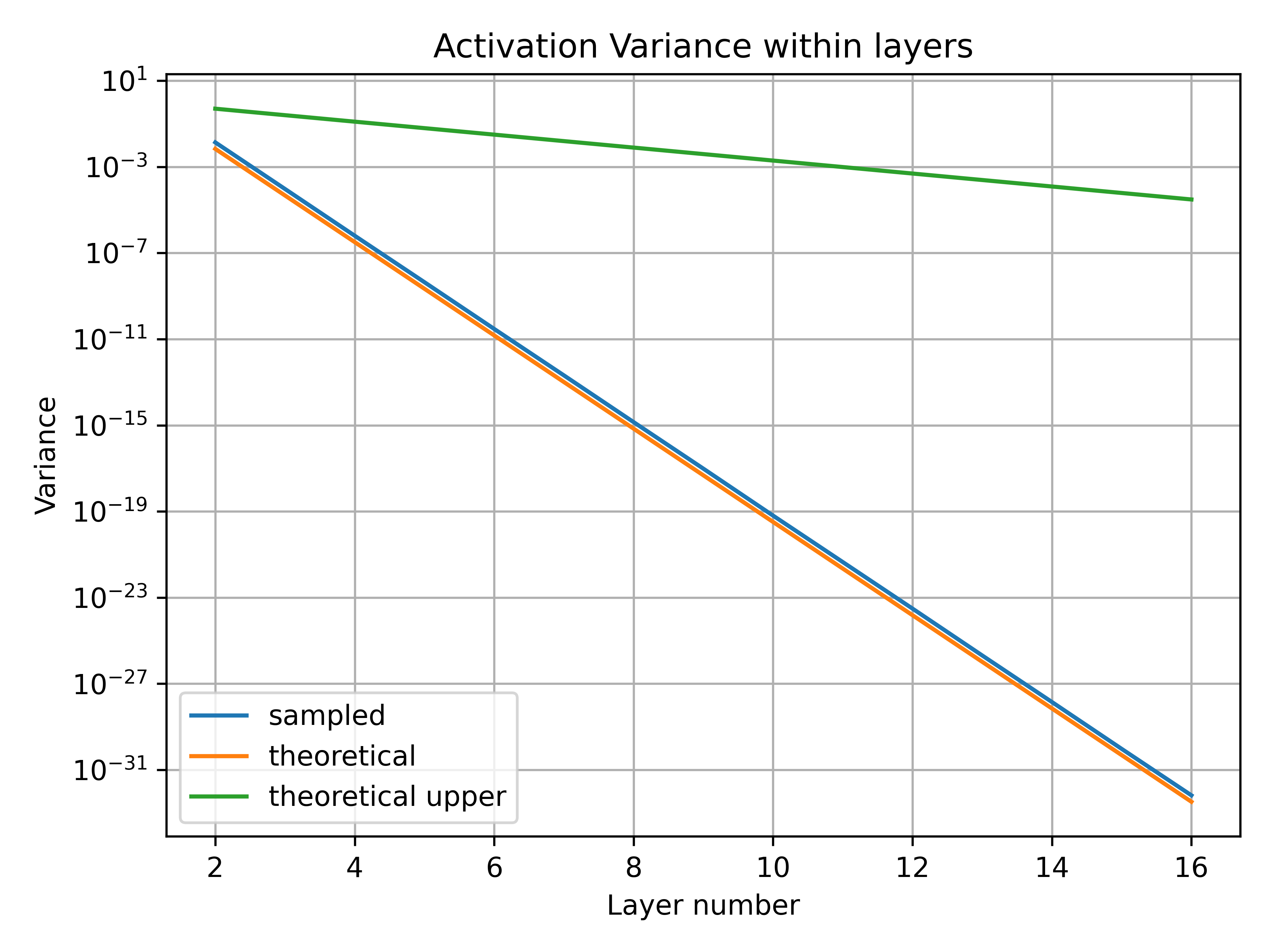}
    \caption{Forward Layer Activation Output Variances}
    \label{fig:VarianceActivation}
\end{figure}

\par
However, once the bias term was set to Eq. \ref{eqn:BiasInitialization}, the output distribution variance produced better results, as depicted in Figure \ref{fig:HistVarianceActivationInit}. This demonstrated that the forward variance output appropriately became one, illustrated in Figure. \ref{fig:VarianceActivationInit}.

\begin{figure}
    \centering
    \includegraphics[width=0.9\columnwidth]{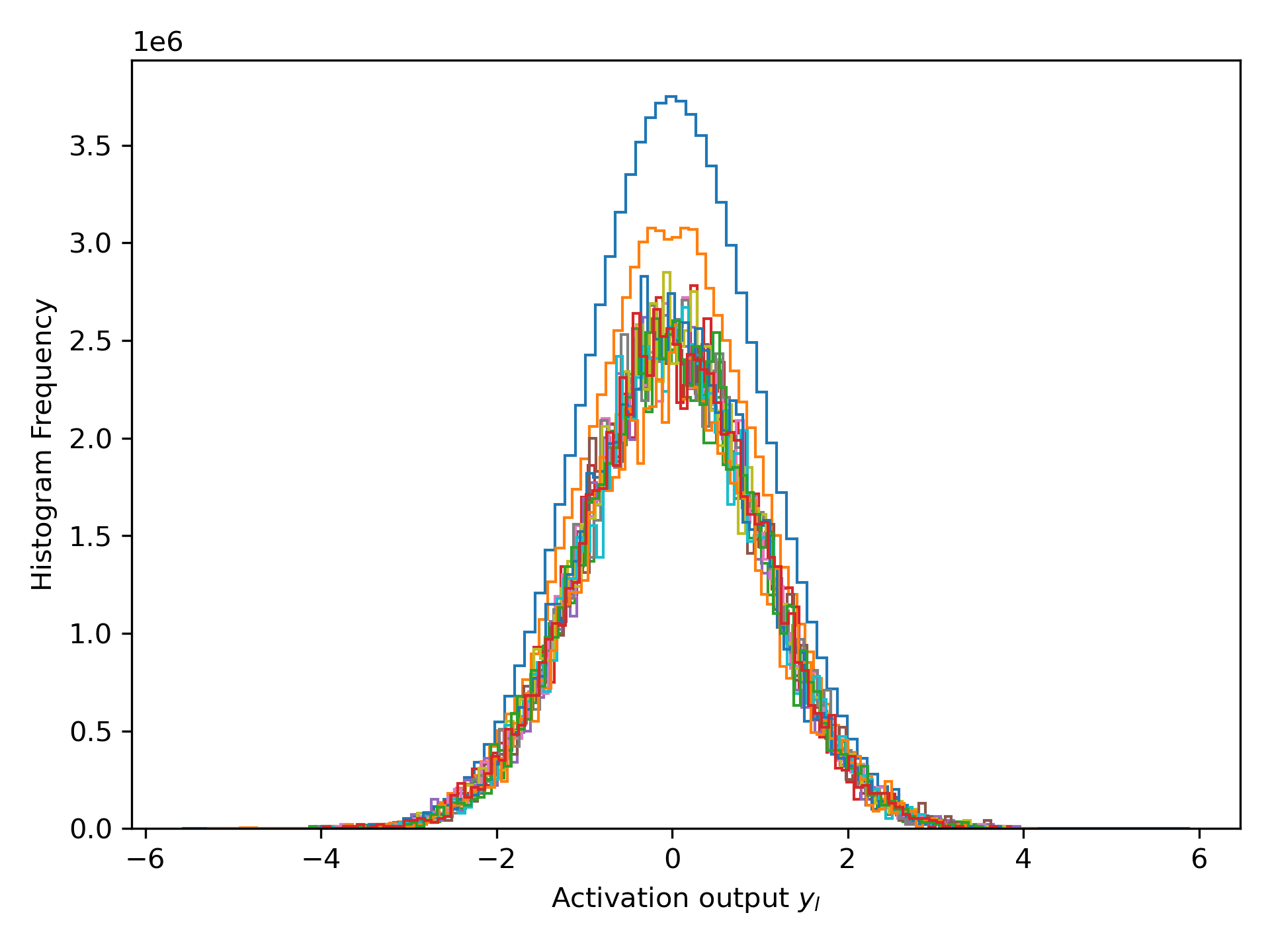}
    \caption{Forward Layer Activation Output with Bias}
    \label{fig:HistVarianceActivationInit}
\end{figure}

\begin{figure}
    \centering
    \includegraphics[width=0.9\columnwidth]{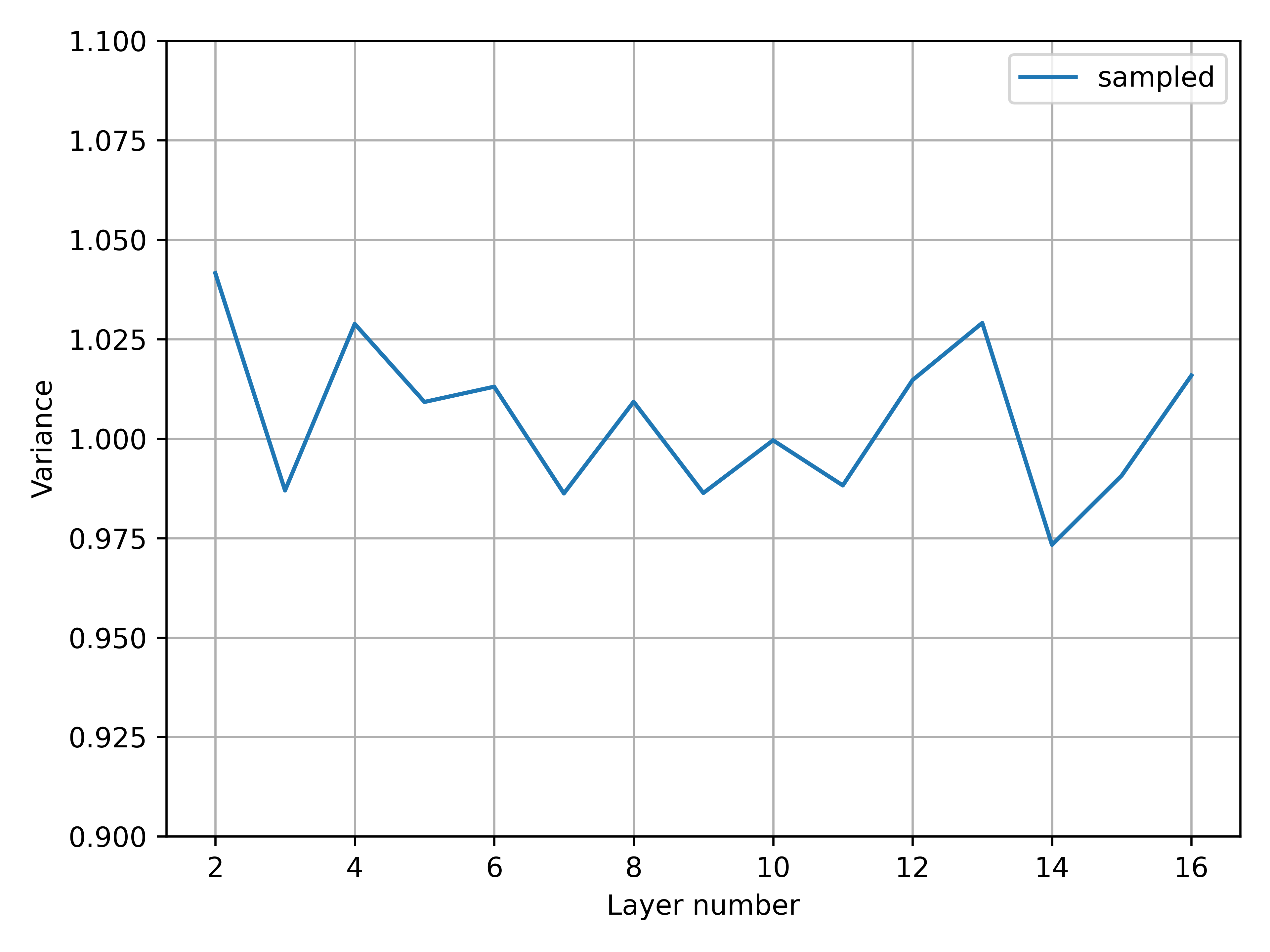}
    \caption{Forward Layer Activation Output Variances with Bias}
    \label{fig:VarianceActivationInit}
\end{figure}

\subsection{Backward-propagation}

Backward propagation was very similar to forward propagation. This time, since the bias term was removed due to the gradient, the output layer distribution formula yielded the same result as in Kaiming's paper.

Instead of the previous forward-propagation equation, the layer was rearranged to:
\begin{align}
    \Delta \boldsymbol{x_l} &= \Tilde{W}_l  \Delta \boldsymbol{y_l} . 
\end{align}
Where $\Delta \boldsymbol{x_l}$ and $\Delta \boldsymbol{y_l}$ denote the gradients $\frac{\partial \mathcal{E}}{\partial \boldsymbol{x}}$ and $\frac{\partial \mathcal{E}}{\partial \boldsymbol{y}}$  respectively.
\begin{theorem}
    If the activation is a ReLU, the gradient of the feedforward network's layers will follow the recursive definition \citet{He2015}:
    \begin{align}
        \Var{\Delta x_2} &= \Var{\Delta x_{L + 1}}\left(\prod_{l = 2}^L \frac{1}{2} \hat{n}_l \Var{w_l} \right).
    \end{align}
\end{theorem}
Given the work previously done in \ref{eqn:exactVariance}, $\Var{w_l}$ was a known quantity and substituted to generate the expected gradient distribution variance:

\begin{align}
     \Var{\Delta x_2} &= \Var{\Delta x_{L + 1}}\left(\prod_{l = 2}^L \frac{\hat{n}_l}{2 \hat{d}_l + 2 (\hat{n}_l -1) \frac{2 \Gamma \left(\frac{\hat{d}_l+1}{2}\right)}{\sqrt{\pi } \Gamma \left(\frac{\hat{d}_l}{2}\right)} } \right).
\end{align}

Which was \ref{eqn:decayEqn}, but with gradient-based parameters, the layer gradient variance was demonstrated in Figure \ref{fig:GradientVarianceActivationGroup}. 

\begin{figure}
    \centering
    \includegraphics[width=0.9\columnwidth]{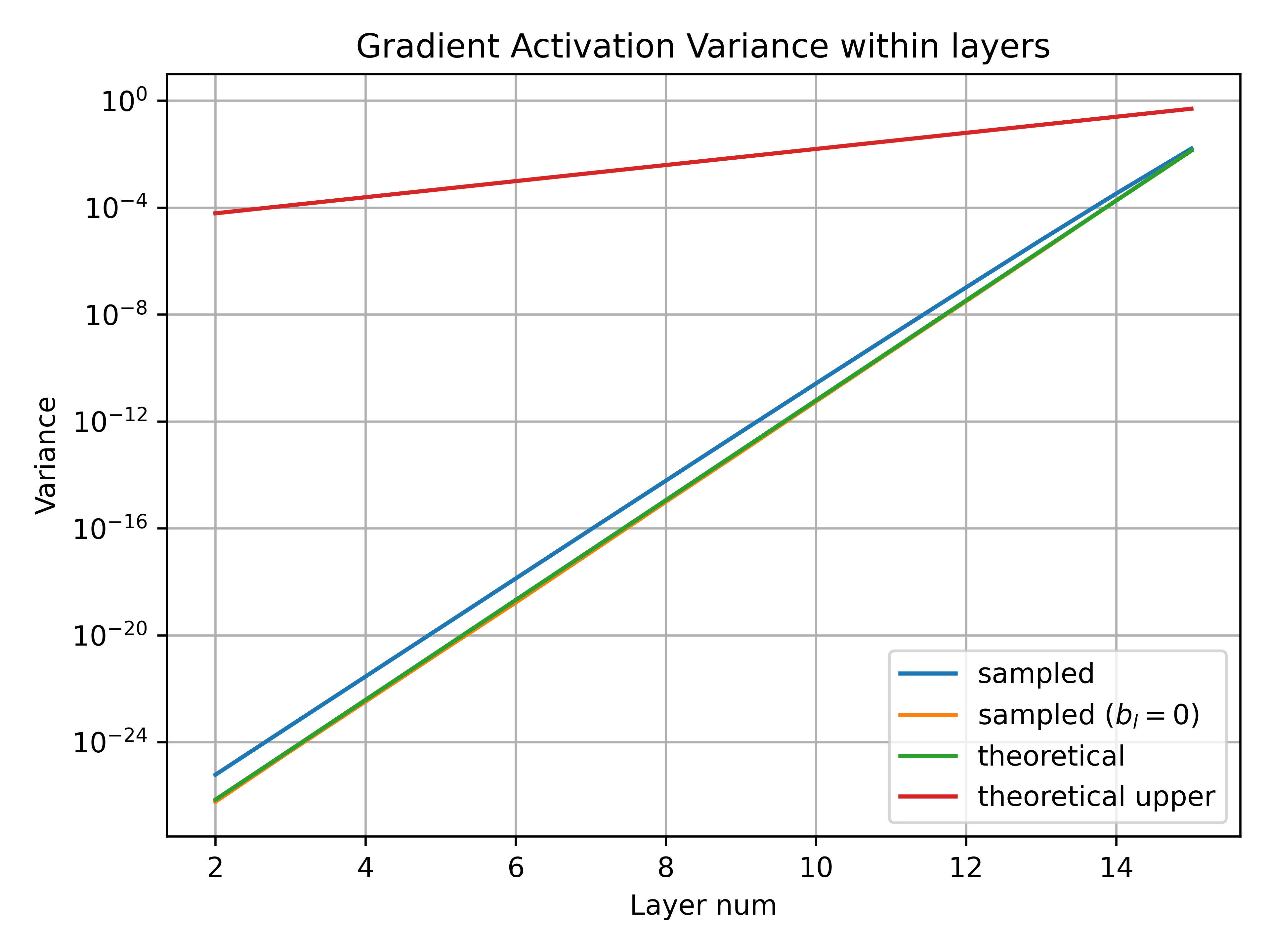}
    \caption{Backwards Layer Activation Output Variances}
    \label{fig:GradientVarianceActivationGroup}
\end{figure}

\par
As noted, the bias term was no longer included, implying that the output distribution would not change regardless of the bias term. This issue was verified by the output gradient distribution of a simple feedforward network with the weight matrix of size of $2048 \times 2048$ in Figures \ref{fig:HistGradVarianceActivation} and \ref{fig:HistGradVarianceActivationB}. 

\begin{figure}
        \centering
    \includegraphics[width=0.9\columnwidth]{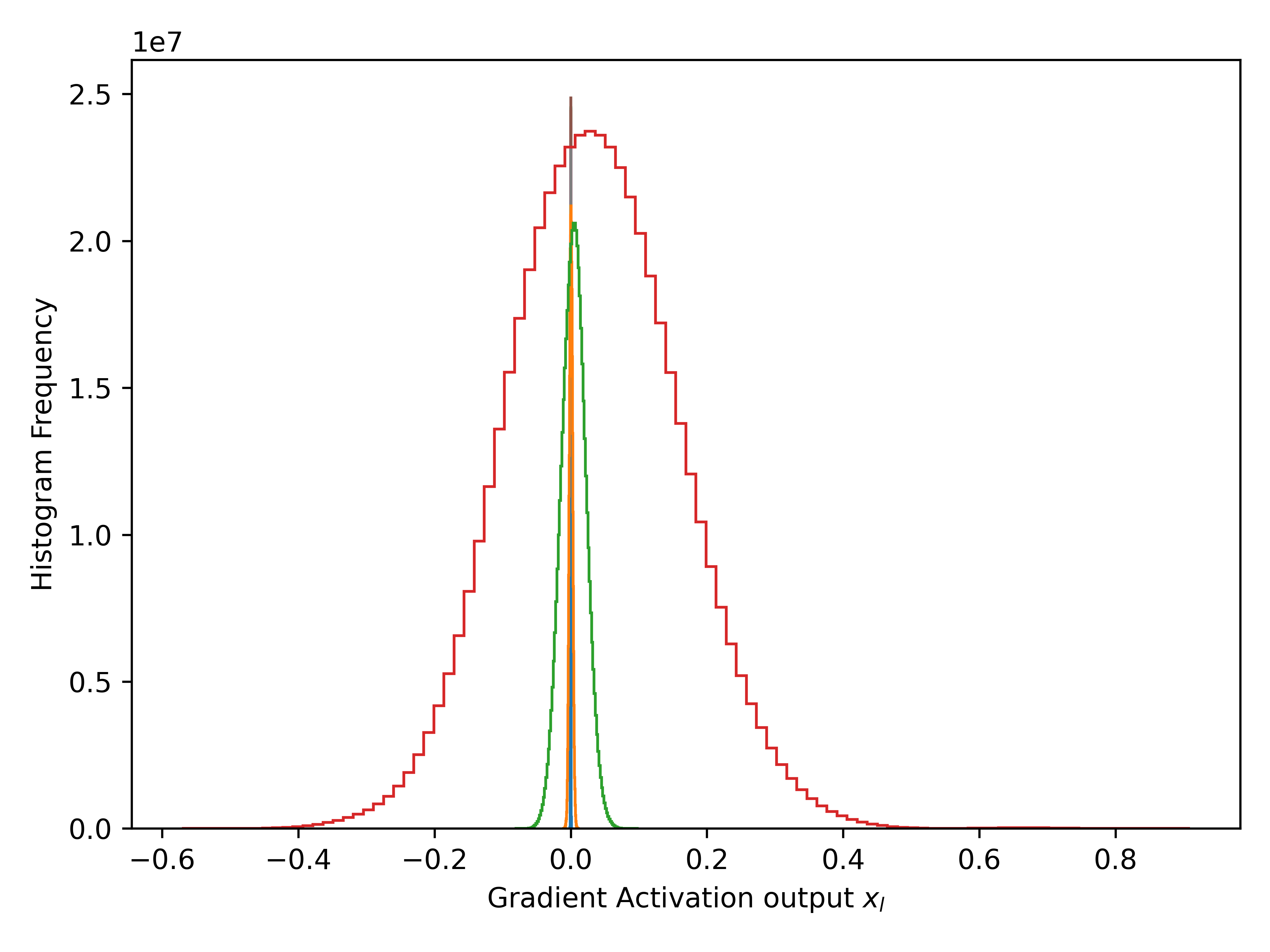}
    \caption{Backwards Layer Activation Output}
    \label{fig:HistGradVarianceActivation}
\end{figure}

\begin{figure}
        \centering
        \includegraphics[width=0.9\columnwidth]{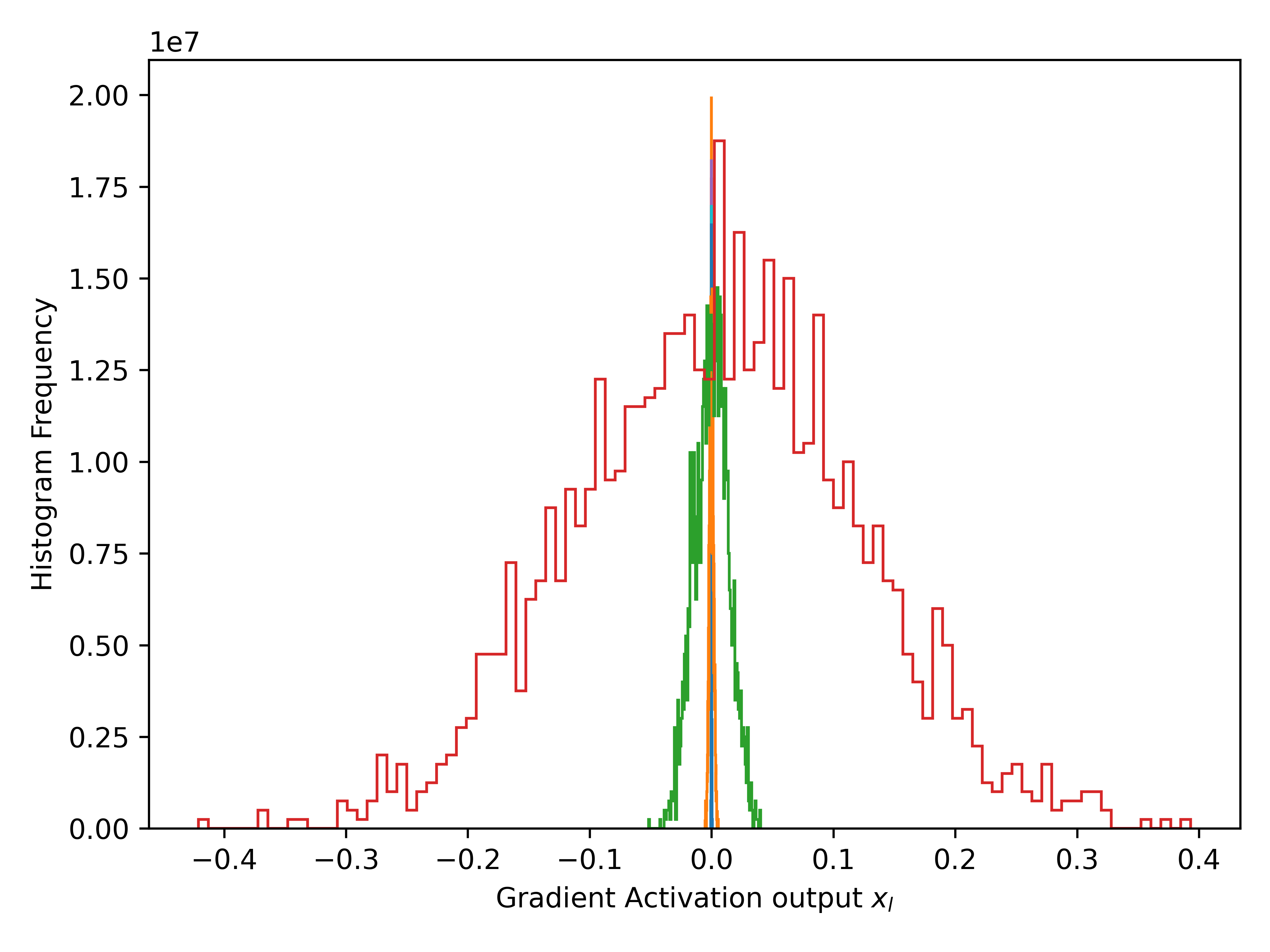}
        \caption{Backwards Layer Activation Output with Bias}
        \label{fig:HistGradVarianceActivationB}
    \end{figure}

So, although setting the bias term for forward propagation helped normalize the output variance, backward propagation cannot be corrected by modifying the network's initialization.

\section{Generalized Normal Initialization}

Interestingly, using a uniform distribution to initialize the system yields the same result as using the normal distribution with output weight variance. To demonstrate this result, the transformed weight variance was formulated using a Generalized Normal Distribution (GND) \citet{Nadarajah2005} for the weight matrix initialization instead of the normal distribution. This was because, from the Generalized Normal Distribution, it is possible to extract multiple other distributions, including but not limited to the Laplace distribution when the shape parameter $\alpha$,  $\alpha=1$, the normal distributions when $\alpha = 2$, and the uniform distribution when $\alpha = \infty$. The PDF and the cumulative distribution function (CDF) are described below \citet{Soury2012, Soury2015}:
\begin{align}
    p(x; \mu, \sigma, \alpha)  &= \frac{\alpha \Lambda  }{2  \Gamma \left(\frac{1}{\alpha }\right)} e^{-\Lambda^\alpha | x - \mu|^{\alpha }}, \label{eqn:GND} \\
    \Lambda &= \frac{\Lambda_0}{\sigma} =  \frac{1}{\sigma} \sqrt{\frac{\Gamma(3 /  \alpha)}{\Gamma(1 /  \alpha)}}, \\
    \boldsymbol{\Phi}(x; \mu, \sigma, \alpha) 
    &= \mathbbm{1}_{x -\mu \ge 0} -  \frac{\text{sign}(x - \mu)}{2 } \nonumber \\ & \qquad \qquad \qquad Q\left(\frac{1}{\alpha}, \Lambda_0^\alpha \left(\frac{\abs{x-\mu}}{\sigma} \right)^\alpha\right),
\end{align}
where $Q(a, z) = \frac{\Gamma(a, z)}{\Gamma(a)}$ is the regularized incomplete upper gamma function, where $\Gamma(a, z)$ is the upper incomplete gamma function,
\begin{align}
    \Gamma(a, z) &= \int_z^\infty t^{a - 1}e^{-t}dt,
\end{align} 
and $\mathbbm{1}_{x> 0}$ is the indicator function. Given that it was previously assumed that $\mu = 0$, this assumption will carry through in the following derivations.

The inverse CDF of this distribution, useful when performing efficient sampling of the distribution, is defined as:

\begin{align}
    \boldsymbol{\Phi}^{-1}(x; \mu, \sigma, \alpha) &= \mu +        \frac{\text{sign}(x - \frac{1}{2})}{\Lambda}Q^{-\frac{1}{\alpha}}\left(\frac{1}{\alpha}, 1-\left|2x-1\right| \right)
\end{align}

where $Q^{-1}(a, s)$ represents the inverse of the regularized incomplete gamma function \citet{DiDonato1986}.

\subsection{Squared Generalized Normal Distribution}

Given the previous distribution in Eq. \ref{eqn:variation_w}, $w_i^2$ needed to be derived. To derive this $\boldsymbol{\Phi}$ was denoted as the CDF of $w_i$, i.e., $\boldsymbol{\Phi}(z) = P(Z < z) = F_Z(z)$, the CDF of $Y = Z^2$ first needed to be calculated, in terms of $\boldsymbol{\Phi}$.
\begin{align}
    F_Y(y)= P(Y \leq y) &= P(Z^2 \leq y)  \nonumber \\
    &= P(-\sqrt{y} \leq Z \leq \sqrt{y}), \text{for y $\ge 0$} \nonumber  \\
    &= \boldsymbol{\Phi}\left(\sqrt{y}\right) - \boldsymbol{\Phi}\left(-\sqrt{y}\right).
\end{align}
To compute the PDF, the CDF's derivative needed to be calculated.
\begin{align}
        f_Y(y) &= \frac{\delta}{\delta y}\left[F_Y(y)\right] = \frac{\delta}{\delta y}\left[\boldsymbol{\Phi}\left(\sqrt{y}\right) - \boldsymbol{\Phi}\left(-\sqrt{y}\right)\right] \nonumber  \\
        &= \frac{\delta}{\delta y}\boldsymbol{\Phi}\left(\sqrt{y}\right) - \frac{\delta}{\delta y}\boldsymbol{\Phi}\left(-\sqrt{y}\right) \nonumber  \\
        &= \varphi\left(\sqrt{y}\right) \frac{1}{2\sqrt{y}} + \varphi\left(-\sqrt{y}\right) \frac{1}{2\sqrt{y}} \nonumber  \\
        &= \frac{1}{2\sqrt{y}} \left(\frac{\alpha \Lambda  e^{-\Lambda^\alpha | \sqrt{y}|^{\alpha }} }{2  \Gamma \left(\frac{1}{\alpha }\right)} + \frac{\alpha \Lambda  e^{-\Lambda^\alpha | -\sqrt{y}|^{\alpha }} }{2  \Gamma \left(\frac{1}{\alpha }\right)} \right)\nonumber  \\
        &= \frac{1}{\sqrt{y}} \frac{\alpha \Lambda  e^{-\Lambda^\alpha | \sqrt{y}|^{\alpha }} }{2  \Gamma \left(\frac{1}{\alpha }\right)} .
\end{align}
The moment-generating function can be derived as:
\begin{align}
    \E{Y^n} &= \int_{0}^\infty y^n f_Y(y) dy \\
    &=\sigma ^{2 n} \Gamma \left(\frac{1}{\alpha }\right)^{n-1} \Gamma
   \left(\frac{3}{\alpha }\right)^{-n} \Gamma \left(\frac{2 n+1}{\alpha
   }\right), \label{eqn:generalized_expectation}
\end{align}
which generated the mean of $\E{w_i^2} = \sigma^2$ and variance $\Var{w_i^2} = \sigma ^4 
 \left(\frac{\Gamma \left(\frac{1}{\alpha }\right) \Gamma \left(\frac{5}{\alpha }\right)}{\Gamma \left(\frac{3}{\alpha }\right)^2} -1\right)$.
\par
Given that the distribution $w_i^2$ has thus been determined, it was also wished to determine what the distribution $\sum_{i=1,j\neq i}^{d_l}w_j^2$ represented.
\par
To find this distribution, the characteristic function (CF) of the Squared Generalized Normal Distribution (SGND) is used, as it is easier to perform the summation in the characteristic equation, which represents the product of characteristic functions, rather than convolving the PDFs.

\subsection{SGND Characteristic Function}

Let $\sigma, \alpha > 0$, and $Y$ be a random variable (RV) following a $SGND(0, \sigma, \mu)$

\begin{theorem}
    The CF of Y, $\E{e^{itY}}$, is given by
    \begin{align}
        \varphi(t; \sigma, \alpha) &= \frac{1}{\Gamma(\frac{1}{\alpha})}{\mathrm{H}}_{1,1}^{1,1}\left[-\iu \Lambda^2 t \left|\genfrac{}{}{0pt}{}{(1-\frac{1}{\alpha},\frac{2}{\alpha})}
{(0, 1)}\right.\right],
    \end{align}
    where $\mathrm{H}^{\cdot,\cdot}_{\cdot,\cdot}[\cdot]$ is the Fox H function (FHF) \citep[Eq. (1.1.1)]{Kilbas2004}.
\end{theorem}
\begin{proof}
    Starting with the definition of the CF and the PDF of the SGND, the CF is defined as
    \begin{align}
        \varphi(t; \sigma, \alpha) &=  \E{e^{itY}} = \int_{\mathbb{R}} e^{\iu ty}f_Y(y)dy \nonumber \\
        &= \frac{\alpha \Lambda}{2  \Gamma \left(\frac{1}{\alpha }\right)} \int_{\mathbb{R_+}}\frac{1}{\sqrt{y}}   e^{-\Lambda^\alpha y^{\frac{\alpha}{2} } } e^{\iu ty} dy.
    \end{align}
    We can find alternative expressions to the exponentials in terms of FHF as \citep[Eq. (2.9.4)]{Kilbas2004}, 
    \begin{align}
        \frac{1}{\beta}y^{\frac{b}{\beta}}e^{-y^{\frac{1}{\beta}}} &= {\mathrm{H}}_{0,1}^{1,0}\left[y\left|\genfrac{}{}{0pt}{}{\emptyline}
{(b, \beta)}\right.\right],
    \end{align}
    which allows us to rewrite the CF integral as the product of two FHF 
    \begin{align}
        \frac{1}{\sqrt{y}}   e^{-\Lambda^\alpha y^{\frac{\alpha}{2} } } &= \frac{2 \Lambda}{\alpha}{\mathrm{H}}_{0,1}^{1,0}\left[\Lambda^2 y\left|\genfrac{}{}{0pt}{}{\emptyline}
{(-\frac{1}{\alpha}, \frac{2}{\alpha})}\right.\right], \\
e^{\iu ty} &= {\mathrm{H}}_{0,1}^{1,0}\left[-\iu ty\left|\genfrac{}{}{0pt}{}{\emptyline}
{(0, 1)}\right.\right],
    \end{align}
    over the positive real numbers, which enables the use of the integral identity defined in \citep[Eq. (2.8.4)]{Kilbas2004}. As a result, the CF is rewritten as
     \begin{align}
        \varphi(t; \sigma, \alpha) &= \frac{\alpha \Lambda}{2  \Gamma \left(\frac{1}{\alpha }\right)}  \frac{2 \Lambda}{\alpha} \int_{0}^{\infty}{\mathrm{H}}_{0,1}^{1,0}\left[\Lambda^2 y\left|\genfrac{}{}{0pt}{}{\emptyline}
{(-\frac{1}{\alpha}, \frac{2}{\alpha})}\right.\right]  \times \nonumber \\
& \quad\quad \quad \quad \quad \quad \quad \quad {\mathrm{H}}_{0,1}^{1,0}\left[-\iu ty\left|\genfrac{}{}{0pt}{}{\emptyline}
{(0, 1)}\right.\right] dy,  \nonumber \\
&= \frac{\Lambda^2}{  \Gamma \left(\frac{1}{\alpha }\right)}  \frac{1}{\Lambda^2} {\mathrm{H}}_{1,1}^{1,1}\left[-\iu \Lambda^2 t \left|\genfrac{}{}{0pt}{}{(1-\frac{1}{\alpha},\frac{2}{\alpha})}
{(0, 1)}\right.\right].
    \end{align}
which completes the derivation of the CF.
\end{proof}
\subsection{Moment Generating Function}
The moment generating function (MGF) can be directly concluded from the Cf by the relation $M(t; \sigma, \alpha ) = \varphi(-\iu t; \sigma, \alpha)$ such that,
\begin{align}
   M(t; \sigma, \alpha ) &= \frac{1}{  \Gamma \left(\frac{1}{\alpha }\right)}   {\mathrm{H}}_{1,1}^{1,1}\left[ \Lambda^2 t \left|\genfrac{}{}{0pt}{}{(1-\frac{1}{\alpha},\frac{2}{\alpha})}
{(0, 1)}\right.\right].
\end{align}
\subsection{Sum of independent SGND random variables}

While it would be interesting to be able to compute the PDF of the generalized sum of $n$ independent SGND random variables and its different characteristics, this would involve an $n$-dimensional Mellin-Barnes integration \citep[Eq. 1.1.2]{Kilbas2004}. These integrations do not have an explicit solution, except for a very limited parameterization of the FHF. Given the CF of a function, the sum of independent SGND random variables (with equal parameters) would be defined as
\begin{align}
    \varphi_n(t; \sigma, \alpha) := \varphi(t; \sigma, \alpha)^n,
\end{align}
with the  $n$-dimensional Mellin-Barnes integral being represented as, 
\begin{align}
    \varphi_n(t; \sigma, \alpha) &= \frac{1}{(2 \pi \iu)^n} \int_{\mathcal{L}^n} \prod_{j = 1}^n \frac{\Gamma(s_j)\Gamma(\frac{1}{\alpha} - \frac{2}{\alpha}s_j)}{\Gamma(1 - \frac{1}{\alpha} + \frac{2}{\alpha}s_j)\Gamma(1 - s_j)} \times \nonumber \\
    & \quad\quad\quad\quad\quad\quad(-i \Lambda^2 t)^{-s_j} d {s_j}.
\end{align}
Given that the derivation of a simplification of is not feasible, taking the inverse Laplace transform of the CF to retrieve the PDF is also not feasible. Instead, given that this paper is only interested in the expectation of this distribution, this would be the $n$-sum of the expectation of a single SGND (Eq. \ref{eqn:generalized_expectation}), 
\begin{align}
    \E{\sum^n_{j = 1}w_j^2} = \sum^n_{j = 1}\E{w_j^2} = n \sigma^2
\end{align}
\subsection{Sum of the absolute value of the product of independent GND}

Finally, the last term that is required to be computed is the sum of the absolute value of the product-independent GND $\sum_{j = 1}^{n_l}\left|\sum_{k = 1}^{d_l - 1}w_{b, k,j} w_{c, k,j} \right|$, where $w_{b, k,j} w_{c, k,j} \sim GND(0, \sigma, \alpha)$. 
\begin{theorem}
    The PDF of the product of two independent IID random variables  $Z = X Y$, where $X, Y \sim GND(0, \sigma, \alpha)$ is annotated as the PGND distribution, is defined as,
    \begin{align}
        f_Z(z) = \frac{\alpha \Lambda^2}{\Gamma(\frac{1}{\alpha})^2}  K_0 (2 \Lambda^\alpha |z|^{\frac{\alpha}{2}}) \label{eqn:productGND}
    \end{align}
\end{theorem}
\begin{proof}
    For two independent random variables with PDF $f_X(x)$, the PDF of the product can be defined as:
    \begin{align}
        f_Z(z) = \int_{-\infty}^\infty f_X(x) f_X\left(\frac{z}{x}\right) \frac{1}{|x|} dx.
    \end{align}
    Given that in this case the $f_X(x)$ is even $p(x) = p(-x)$, it is possible to rewrite the integral as
    \begin{align}
        f_Z(z) = 2 \int_{0}^\infty f_X(x) f_X\left(\frac{z}{x}\right) \frac{1}{|x|} dx.
    \end{align}
    The definition of $f_X(x)$ from Eq. \ref{eqn:GND} is substituted in $f_Z(z)$ and since $x >0$, $|\frac{z}{x}| = \frac{|z|}{x}$ as such
    \begin{align}
        f_Z(z) &= 2 \left(\frac{\alpha \Lambda  }{2  \Gamma \left(\frac{1}{\alpha }\right)} \right)^2 \int_{0}^\infty  e^{-\Lambda^\alpha x^{\alpha }}  e^{-\Lambda^\alpha \left(\frac{|z|}{x}\right)^{\alpha }} \frac{dx}{x}, \nonumber \\
        &= 2 \left(\frac{\alpha \Lambda  }{2  \Gamma \left(\frac{1}{\alpha }\right)} \right)^2 \int_{0}^\infty  \frac{1}{x} e^{-\Lambda^\alpha \left( x^{\alpha } +|z|^\alpha x^{-\alpha} \right)} dx.
    \end{align}
    Which has a known solution given the standard integral formula \citep[Eq. 3.478.4]{def_integrals} which states that for $\Re{\beta} >0, \Re \gamma > 0$,
    \begin{align}
        \int_0^\infty x^{\nu - 1} e^{-\beta x^{p} - \gamma x^{-p}} = \frac{2}{p} \left(\frac{\gamma}{\beta}\right)^{\frac{\nu}{2 p}} K_{\frac{\nu}{p}}\left(2 \sqrt{\beta \gamma}\right).
    \end{align}
    In this case we set $\nu = 0, \beta = \Lambda^\alpha, p = \alpha$ and $\gamma = \Lambda^\alpha |z|^\alpha$ and get the final answer that,
    \begin{align}
        f_Z(z) = \frac{\alpha \Lambda^2}{\Gamma(\frac{1}{\alpha})^2} K_0 (2 \Lambda^\alpha |z|^{\frac{\alpha}{2}})
    \end{align}
\end{proof}
\subsection{Moment generating function of the absolute value of the PGND}

Given the derived PDF of the product of IID GND distributions, to continue, it is desired to compute the general n sum of this product distribution. \ref{eqn:productGND}. However, to start, we compute the moments of its absolute value for later use.
\begin{theorem}[MGF of the Absolute PGND]
    If the $Z \sim PGND(\alpha, \sigma)$, then the moment generating function's absolute value is its Mellin transform, 
    \begin{align}
        \mathcal{M}\{f_Z\}(s) = \frac{\Lambda^{2 - 2s}}{ \Gamma\left(\frac{1}{\alpha}\right)^2} \Gamma\left(\frac{s}{\alpha}\right)^2 \label{eqn:APGNDMomentGenerating Function}
    \end{align}
\end{theorem}
\begin{proof}
    The Mellin transform of $f$ is defined by
    \begin{align}
        \mathcal{M}\{f\}(s) = \int_0^\infty z^{s - 1}f(z)fz
    \end{align}
    Given that $f_Z$ (Eq. \ref{eqn:productGND}) is an even function, we can look for $z > 0$, which can be written as,
    \begin{align}
         f_Z(z) &= \frac{\alpha \Lambda^2}{\Gamma(\frac{1}{\alpha})^2} K_0 (2 \Lambda^\alpha z^{\frac{\alpha}{2}}).
    \end{align}
    Introduce the substitution,
    \begin{align}
        u=2\,\Lambda^\alpha\,z^{\alpha/2}\quad\Longrightarrow\quad z=\Bigl(\frac{u}{2\,\Lambda^\alpha}\Bigr)^{2/\alpha},
    \end{align}
    with
    \begin{align}
        dz=\frac{2}{\alpha}\Bigl(\frac{u}{2\,\Lambda^\alpha}\Bigr)^{2/\alpha-1}\frac{du}{2\,\Lambda^\alpha}
=\frac{1}{\alpha\,\Lambda^\alpha}\Bigl(\frac{u}{2\,\Lambda^\alpha}\Bigr)^{2/\alpha-1}du
    \end{align}
    with the understanding that the complete transform on $\mathbb{R}$ can be recovered from the even symmetry, the Mellin transform becomes,
    \begin{align}
\mathcal{M}\{f_z\}(s)
&=\frac{2\alpha\,\Lambda^2}{\Gamma\left( \frac{1}{\alpha} \right)^2}\int_0^\infty z^{s-1}\;K_0\Bigl(2\,\Lambda^\alpha\,z^{\frac{\alpha}{2}}\Bigr)dz, \nonumber \\
&=\frac{2\alpha\,\Lambda^2}{\Gamma\left( \frac{1}{\alpha} \right)^2}\int_0^\infty \Bigl(\frac{u}{2\,\Lambda^\alpha}\Bigr)^{\frac{2(s-1)}{\alpha}}
\times \nonumber \\ &\qquad \qquad \qquad \qquad K_0(u)
\frac{1}{\alpha\,\Lambda^\alpha}\Bigl(\frac{u}{2\,\Lambda^\alpha}\Bigr)^{\frac{2}{\alpha}-1}du, \nonumber \\
&=\frac{2\Lambda^{2-\alpha}}{\Gamma\left( \frac{1}{\alpha} \right)^2}\Bigl(2\,\Lambda^\alpha\Bigr)^{-2s/\alpha+1}
\int_0^\infty u^{\frac{2s}{\alpha}-1}\,K_0(u)\,du.
\end{align}
Which also has a standard integral formula \citep[Eq. 6.561.16]{def_integrals_special} for $\Re\{\mu + 1 \pm \nu\} \ge 0, \Re a > 0$,
\begin{align}
    \int_0^{\infty}x^\mu K_\nu (a x)dx &= 2^{\mu -1}a^{-\mu - 1} \Gamma\left(\frac{1 + \mu + \nu}{2}\right) \nonumber \\ & \qquad \qquad\qquad\qquad \Gamma\left(\frac{1 + \mu - \nu}{2}\right) \nonumber
\end{align}
where in this case $a = 1, \nu = 0$ and $\mu = \frac{2s}{\alpha} - 1$, which results in ,
\begin{align}
    \mathcal{M}\{f_Z\}(s)
&=\frac{2\Lambda^{2-\alpha}}{\Gamma\left( \frac{1}{\alpha} \right)^2}\,\Bigl(2\,\Lambda^\alpha\Bigr)^{-2s/\alpha+1}\,2^{\frac{2s}{\alpha}-2}\,\Gamma\left(\frac{s}{\alpha}\right)^2 \nonumber \\
&=\frac{\Lambda^{2-2s}}{\,\Gamma\left( \frac{1}{\alpha} \right)^2}\,\Gamma\left(\frac{s}{\alpha}\right)^2.
\end{align}
\end{proof}
\subsection{PGND Characteristic Function}

Let $\sigma > 0$ and $\alpha \ge 2$, and $Y$ be a random variable (RV) following a $PGND(\sigma, \mu)$

\begin{theorem}
    The CF of Y, $\E{e^{itY}}$, is given by
    \begin{align}
        \varphi(t; \sigma, \alpha) &= \frac{\sqrt{\pi} \Lambda^2}{2 \Gamma(\frac{1}{\alpha})^2}  {\mathrm{H}}_{2,2}^{1,2}\left[\Lambda^{-2} \frac{t}{2} \left|\genfrac{}{}{0pt}{}{(1 - \frac{1}{\alpha}, \frac{1}{\alpha}), (1 - \frac{1}{\alpha}, \frac{1}{\alpha})}
{(0, \frac{1}{2}), (\frac{1}{2}, \frac{1}{2})}\right.\right].
    \end{align}
\end{theorem}
\begin{proof}
    Starting with the definition of the CF and the PDF of the PGND, the CF is defined as an even function
    \begin{align}
        \varphi(t; \sigma, \alpha) &=  \E{e^{itY}} = \int_{\mathbb{R}} e^{\iu ty}f_Y(y)dy, \nonumber \\
        &=  \frac{\alpha \Lambda^2}{\Gamma(\frac{1}{\alpha})^2} \int_{\mathbb{R}}K_0 (2 \Lambda^\alpha |y|^{\frac{\alpha}{2}}) e^{\iu ty} dy, \\
        &=  \frac{2 \alpha \Lambda^2}{\Gamma(\frac{1}{\alpha})^2} \int_{0}^{\infty}K_0 (2 \Lambda^\alpha y^{\frac{\alpha}{2}}) \cos(t y) dy.
    \end{align}
    We can find alternative expressions to the exponentials in terms of FHF as \citep[Eq. (2.9.8), (2.9.19)]{Kilbas2004}, 
    \begin{align}
        \cos(x) &= \sqrt{\pi}{\mathrm{H}}_{0,2}^{1,0}\left[\frac{x^2}{4}\left|\genfrac{}{}{0pt}{}{\emptyline}
{(0, 1), (\frac{1}{2}, 1)}\right.\right], \\
&= \frac{\sqrt{\pi}}{2}{\mathrm{H}}_{0,2}^{1,0}\left[\frac{x}{2}\left|\genfrac{}{}{0pt}{}{\emptyline}
{(0, \frac{1}{2}), (\frac{1}{2}, \frac{1}{2})}\right.\right], \\
        K_0(x) &= \frac{1}{2} {\mathrm{H}}_{0,2}^{2,0}\left[\frac{x^2}{4}\left|\genfrac{}{}{0pt}{}{\emptyline}
        {\left(0, 1 \right), \left(0, 1 \right)}\right.\right], \\
        K_0(c x^\nu) &= \frac{1}{4 \nu} {\mathrm{H}}_{0,2}^{2,0}\left[\left(\frac{c}{2}\right)^{\frac{1}{\nu}} x\left|\genfrac{}{}{0pt}{}{\emptyline}
        {\left(0, \frac{1}{2 \nu} \right), \left(0, \frac{1}{2 \nu} \right)}\right.\right],
    \end{align}
    which allows us to rewrite the CF integral as the product of two FHF 
    \begin{align}
        K_0 (2 \Lambda^\alpha y^{\frac{\alpha}{2}}) &= \frac{1}{2 \alpha}{\mathrm{H}}_{0,2}^{2,0}\left[\Lambda^2 y \left|\genfrac{}{}{0pt}{}{\emptyline}
        {(0, \frac{1}{\alpha}), (0, \frac{1}{\alpha})}\right.\right], \\
\cos(ty) &=  \frac{\sqrt{\pi}}{2}{\mathrm{H}}_{0,2}^{1,0}\left[\frac{ty}{2}\left|\genfrac{}{}{0pt}{}{\emptyline}
{(0, \frac{1}{2}), (\frac{1}{2}, \frac{1}{2})}\right.\right],
    \end{align}
    over the positive real numbers, which enables the use of the integral identity defined in \citep[Eq. (2.8.4)]{Kilbas2004}. As a result, the CF is rewritten as
     \begin{align}
        \varphi(t; \sigma, \alpha) &= \frac{2 \alpha \Lambda^2}{\Gamma(\frac{1}{\alpha})^2}  \frac{\sqrt{\pi} }{4 \alpha} \int_{0}^{\infty} {\mathrm{H}}_{0,2}^{1,0}\left[\frac{ty}{2}\left|\genfrac{}{}{0pt}{}{\emptyline}
{(0, \frac{1}{2}), (\frac{1}{2}, \frac{1}{2})}\right.\right]  \times \nonumber \\
& \quad\quad \quad \quad \quad \quad \quad \quad {\mathrm{H}}_{0,2}^{2,0}\left[\Lambda^2 y \left|\genfrac{}{}{0pt}{}{\emptyline}
        {(0, \frac{1}{\alpha}), (0, \frac{1}{\alpha})}\right.\right] dy,  \nonumber \\
&=  \frac{\sqrt{\pi} \Lambda^2}{2 \Gamma(\frac{1}{\alpha})^2}  {\mathrm{H}}_{2,2}^{1,2}\left[\Lambda^{-2} \frac{t}{2} \left|\genfrac{}{}{0pt}{}{(1 - \frac{1}{\alpha}, \frac{1}{\alpha}), (1 - \frac{1}{\alpha}, \frac{1}{\alpha})}
{(0, \frac{1}{2}), (\frac{1}{2}, \frac{1}{2})}\right.\right].
    \end{align}
which completes the derivation of the CF. However, this is only valid for $\alpha \ge 2$ as otherwise, the numerator's parameters would not enable a convergent FHF.
\end{proof}

Even if the CF derived could function for $\alpha >0$ instead of the limited $\alpha \ge 2$, the issue remains that computing the sum and its expectation from the CF would not be feasible. Thus, we would have to determine an approximation to the desired expectation instead.

\subsection{Upper bound}

To derive the upper bound of the expectation derived in Eq. \ref{eqn:expectedValueVariance}, one can show that

\begin{theorem}[Lower bound for IID symmetric sums]
  \label{thm:lower-bound}
  Let $\{X_i\}_{i=1}^\infty$ be an independent and identically distributed sequence of real random variables with
  \begin{align}
  \E{X_i}=0,\quad \E{|X_i|}<\infty,
  \end{align}
  and $X_i$ PDF is given by
  \begin{align}
  f_{X_1}(z)= \frac{\alpha \Lambda^2}{\Gamma(1/\alpha)^2}\;K_0\Bigl(2\Lambda^\alpha |z|^{\alpha/2}\Bigr),\quad z\in\mathbb{R},
  \end{align}
 Then, for the partial sums
  \begin{align}
  S_n = \sum_{i=1}^n X_i,
  \end{align}
  the following lower bound holds:
  \begin{align}
  \E{|S_n|} \ge \sqrt{n}\,\frac{\Gamma\left(\frac{2}{\alpha}\right)^2}{\Lambda^2 \Gamma\left(\frac{1}{\alpha}\right)^2}  =  \sqrt{n}\frac{\sigma ^2 \Gamma \left(\frac{2}{\alpha }\right)^2}{\Gamma \left(\frac{1}{\alpha }\right) \Gamma \left(\frac{3}{\alpha }\right)}.
  \end{align}
\end{theorem}

\begin{proof}
  Since the random variables $\{X_i\}$ are IID with $E[X_i]=0$ and a finite first absolute moment, it is known from the Central Limit Theorem that the typical fluctuation of $S_n$ is of order $\sqrt{n}$. More precisely, by symmetry and scaling, one expects,
  \begin{align}
  \E{|S_n|}\sim \sqrt{n}\,\E{|X|}.
  \end{align}
  For the given PDF,
  \begin{align}
  f_{X}(z)= \frac{\alpha \Lambda^2}{\Gamma(1/\alpha)^2}\;K_0\Bigl(2\Lambda^\alpha |z|^{\alpha/2}\Bigr),
  \end{align}
  a direct evaluation we derive the $E|X|$ as derived from the moment generating function Eq. \ref{eqn:APGNDMomentGenerating Function} evaluated at $s = 2$.
  \begin{align}
  \E{|X|}= \int_{-\infty}^{\infty}|z|f_{X}(z)\,dz = \frac{\Gamma\left(\frac{2}{\alpha}\right)^2}{\Lambda^2 \Gamma\left(\frac{1}{\alpha}\right)^2} 
  \end{align}

  Because the $X_i$ are independent and identically distributed, the sum $S_n$ has the scaling property,
  \begin{align}
  S_n \stackrel{d}{=} \sqrt{n}\,X_1,
  \end{align}
  at least asymptotically. In the case of the PDF above, the equality
  \begin{align}
  \E{|S_n|} = \sqrt{n}\,\E{|X|}
  \end{align}
  holds.
\end{proof}
Thus upper bound of $ \Var{\bar{w}_i}$, \ref{eqn:expectedValueVariance}, was calculated as:
\begin{small}
\begin{align}
    \Var{\bar{w}_i} &= \frac{\E{w_i^2}}{\E{w_i^2} + \E{\sum_{j = 1, j \neq i}^{d_l} w_j^2} + \E{\sum_{j = 1, j \neq i}^{n_l}\left|\sum_{k = 1}^{d_l }w_b w_c \right|}} \nonumber \\
    &\leq \frac{ \sigma^2}{ \sigma^2 + ( d_l - 1)  \sigma^2 + (n_l  - 1) \sqrt{d_l} \frac{\sigma ^2 \Gamma \left(\frac{2}{\alpha }\right)^2}{\Gamma \left(\frac{1}{\alpha }\right) \Gamma \left(\frac{3}{\alpha }\right)}} \nonumber\\
    &\leq \frac{ 1}{ d_l  + (n_l  - 1) \sqrt{d_l} \frac{ \Gamma \left(\frac{2}{\alpha }\right)^2}{\Gamma \left(\frac{1}{\alpha }\right) \Gamma \left(\frac{3}{\alpha }\right)}}
     \label{eqn:varianceEstimation}
\end{align}
\end{small}
\subsection{Lower bound}

To instead find a lower bound of the expectation, we can instead find an upper bound on the expectation of $\E{|\sum^n X_i|}$ where $X_i \sim PGND(\alpha, \sigma)$. The expected value could be represented by the following approximation using the Cauchy–Schwarz inequality:
\begin{align}
    \E{\sum_{j = 1, j \neq i}^{n_l} \abs{\sum_{k = 1}^{d_l}w_a w_b}} &= (n_l  - 1) \E{\abs{\sum_{k = 1}^{d_l}w_a w_b}} \\
                                                           &\leq (n_l  - 1) \sqrt{\E{\left(\sum_{k = 1}^{d_l}w_a w_b\right)^2}}.
\end{align}
When examining the inner expected value, the following was achieved, assuming that the variables were IID with zero means:
\begin{align}
&= \E{\left(\sum_{k = 1}^{d_l}w_a w_b\right)^2} \\
   &=\E{\sum^{d_l}\left(w_a w_b\right)^2 + \sum \left(w_a w_b\right)_i\left(w_a w_b\right)_j} \\
   &=d_l\E{\left(w_a w_b\right)^2} + \sum \E{\left(w_a w_b\right)_i}\E{\left(w_a w_b\right)_j} \\
   &= d_l\E{w_a^2}\E{ w_b^2} = d_l\Var{w_l}^2.
\end{align}
Which thus returned, given that the variance of $\Var{w_l} = \sigma^2$:
\begin{align}
    \E{\sum_{j = 1}^{n_l - 1} \abs{\sum_{k = 1}^{d_l}w_a w_b}} &\leq (n_l  - 1) \sqrt{d_l\Var{w_l}^2} \\
    &= (n_l  - 1) \sqrt{d_l} \Var{w_l} \\
    &= (n_l  - 1) \sqrt{d_l} \sigma^2.
\end{align}
Thus $ \Var{\bar{w}_i}$, \ref{eqn:expectedValueVariance}, was calculated as:
Thus upper bound of $ \Var{\bar{w}_i}$, \ref{eqn:expectedValueVariance}, was calculated as:
\begin{small}
\begin{align}
    \Var{\bar{w}_i} &\leq \frac{\E{w_i^2}}{\E{w_i^2} + \E{\sum_{j = 1, j \neq i}^{d_l} w_j^2} + \E{\sum_{j = 1, j \neq i}^{n_l}\left|\sum_{k = 1}^{d_l }w_b w_c \right|}} \nonumber\\
    &\leq \frac{ \sigma^2}{ \sigma^2 + ( d_l - 1)  \sigma^2 + (n_l  - 1) \sqrt{d_l} \sigma^2}\nonumber \\
    &\leq \frac{ 1}{ d_l  + (n_l  - 1) \sqrt{d_l}}
     \label{eqn:varianceEstimationUpper}
\end{align}
\end{small}

\subsection{Variance Bounding}

Similarly to the derivation for the Gaussian distribution, it can be seen that the scaling factor $\sigma$ was absent from the distribution's output variance, again implying that the initial variance of the weight does not affect the output variance. In contrast, only the matrix size $n_l$ and $d_l$ and the shape parameter $\alpha$ affect the output resultant distribution. Given the variance estimate in \ref{eqn:varianceEstimation}, it was possible to recover a variety of variance estimates based on different distribution initializations; for the Normal distribution ($\beta = 2$), the Laplace distribution ($\beta = 1$), and the uniform distribution ($\beta = \infty$):
\begin{align}
    \Var{\bar{w}_i}\rvert_{\beta = 1} &\leq \frac{ 1}{ d_l  +\frac{1}{2} (n_l  - 1) \sqrt{d_l} }, \\
    \Var{\bar{w}_i}\rvert_{\beta = 2} &\leq \frac{ 1}{ d_l  +\frac{2}{\pi} (n_l  - 1) \sqrt{d_l} } \label{eqn:gaussianGeneralUpperBound},\\
    \Var{\bar{w}_i}\rvert_{\beta = \infty} &\leq \frac{ 1}{ d_l  + \frac{3}{4} (n_l  - 1) \sqrt{d_l}}.
\end{align}
%
%
Given the previous results, when looking at the multilayer layer initialization formulation, it resulted in the following variance upper bound by substituting the multilayer output variance factor:
\begin{align}
  \frac{n_l}{2}\Var{\bar{w}_i} \leq\frac{ n_l}{ 2d_l  + 2(n_l  - 1) \sqrt{d_l} \frac{ \Gamma \left(\frac{2}{\alpha }\right)^2}{\Gamma \left(\frac{1}{\alpha }\right) \Gamma \left(\frac{3}{\alpha }\right)}}.
\end{align}
Given a relatively standard assumption for linear networks that the dimensions of $d_l$ and $n_l$ are relatively close to each other, we can set, for the sake of simplicity, $d_l = n_l$; as such, we get 
\begin{align}
  \frac{n_l}{2}\Var{\bar{w}_i} &\leq \frac{n_l}{2 \left(\frac{\Gamma \left(\frac{2}{\alpha }\right)^2 \left(n_l-1\right) \sqrt{n_l}}{\Gamma
   \left(\frac{1}{\alpha }\right) \Gamma \left(\frac{3}{\alpha }\right)}+n_l\right)} \leq \frac{1}{2}
\end{align}

Thus, causing an upper-bounded exponential decaying rate for the multilayer from \ref{eqn:layerVariance}, which was bounded as:
\begin{align}
   \Var{y_L} &\leq \prod_{l = 2}^L \left( \frac{1}{2}\right) \Var{y_1} + \frac{\Var{b_l}}{1 - \frac{1}{2}}  \\
   &= 2^{-(L - 1)} \Var{y_1} + 2\Var{b_l}, \\
   \lim_{L \to \infty} \Var{y_L} &= 2 \Var{b_l}.
\end{align}
This demonstrated that, as long as the weight was initialized using a generalized normal distribution variant, the final layer variance would be proportional only to the bias's variance, and the initial input would not convey any information to the deeper layers.

\section{Experiments}

To demonstrate the problem, we train a linear SDP-based network from \citet{Araujo2023} on the Covertype dataset \citet{Blackard1998Covertype} (split 80/20 by random sampling), which predicts forest cover type from cartographic variables only, with input features that are categorical or integer-valued and no images. We then try fitting the network with a hidden dimension of $64$ and a batch size of $64$. The networks were all trained with the Adam optimizer (using default PyTorch values) with a learning rate of 0.001 each. We tested whether the number of internal layers (5, 15, or 30) and whether the bias was initialized to zero or not; we report mean $\pm$ standard deviation over 10 seeds, and show training loss trajectories.
\begin{figure}
    \centering
    \includegraphics[width=1\linewidth]{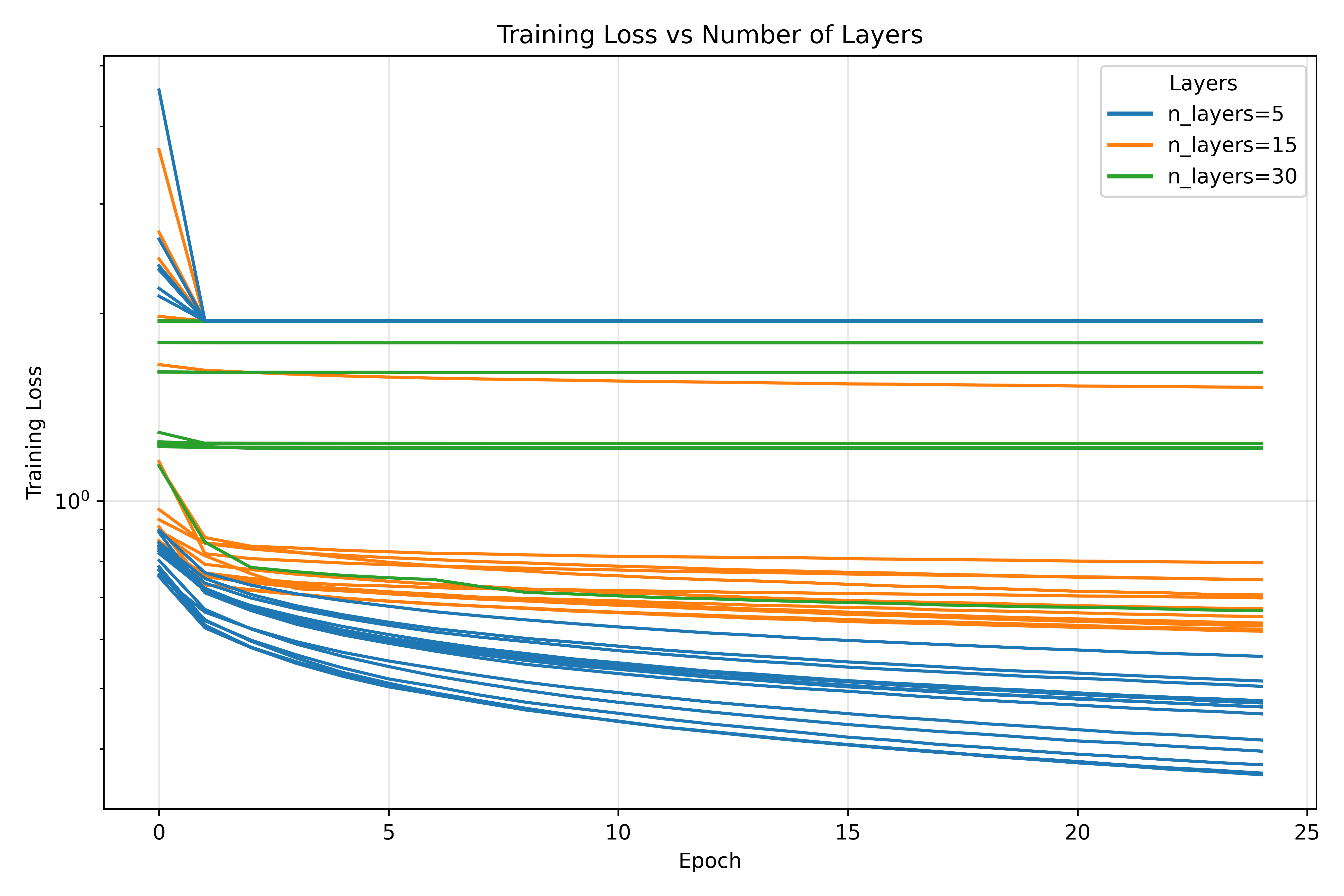}
    \caption{All the trained SLL Lipschitz networks' loss over epochs for training on the Covertype dataset, colored by number of layers}
    \label{fig:exp_all_loss}
\end{figure}

\begin{figure*}
    \centering
    \includegraphics[width=0.75\linewidth]{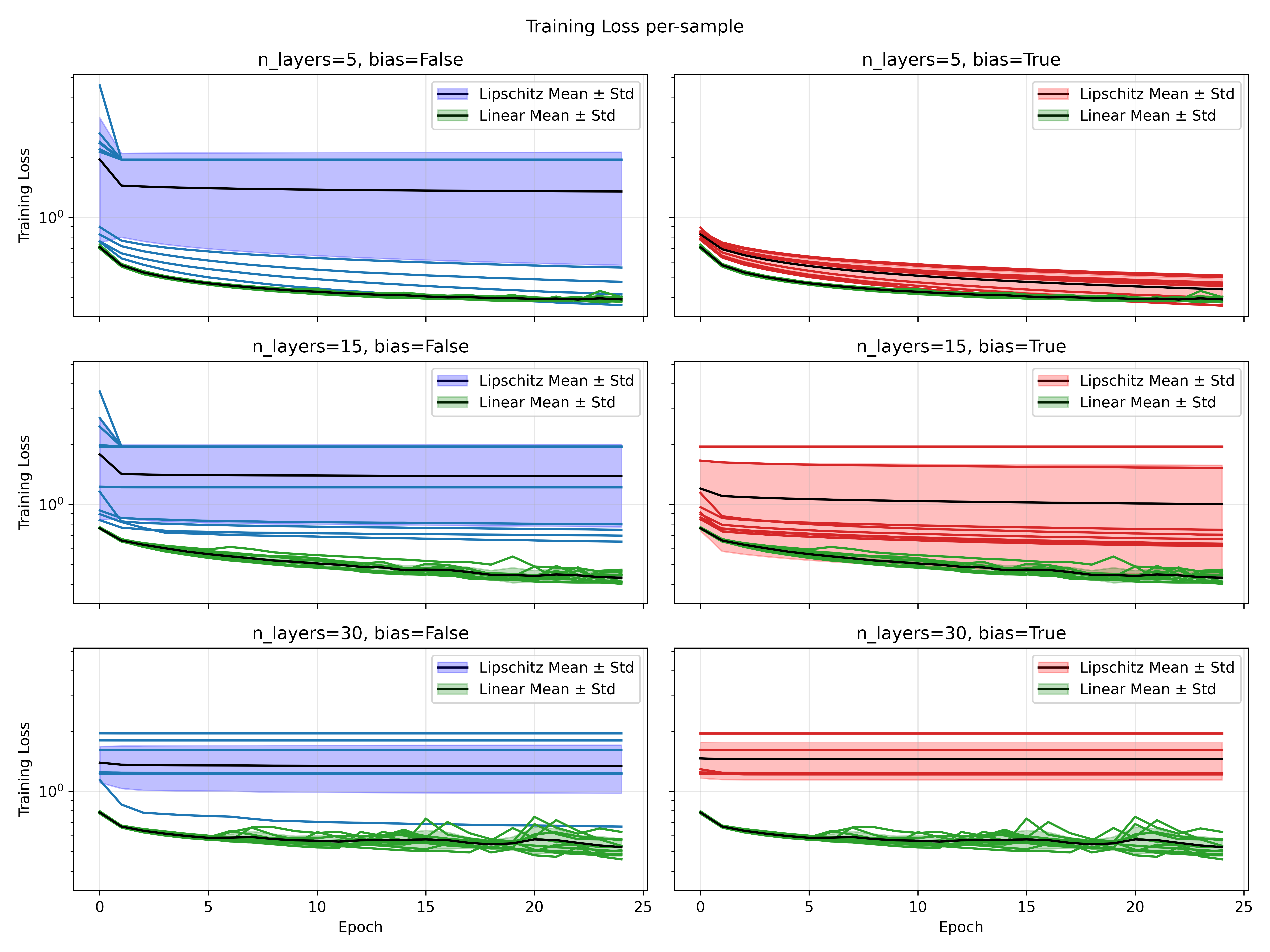}
    \caption{The trained SLL Lipschitz networks' loss over epochs for training on the Covertype dataset, grouped by number of layers and bias initialization. The SLL network is compared with a standard feedforward network with equivalent parameters in green}
    \label{fig:exp_group_loss}
\end{figure*}

As demonstrated in Figures \ref{fig:exp_all_loss} and \ref{fig:exp_group_loss}, the network's training on this simple dataset is unstable. We can clearly see that as the number of layers increases, the training stagnates and worsens; this is also corroborated by \ref{fig:exp_succ_training}, which aggregates success statistics for the network, showing that the training actually has a non-zero dynamic by checking whether the training loss changes (e.g. >1\% reduction from epoch 2 and onward). If the training loss is constant, then it is deemed that the system did not successfully train. 

\begin{figure*}
    \centering
    \includegraphics[width=0.75\linewidth]{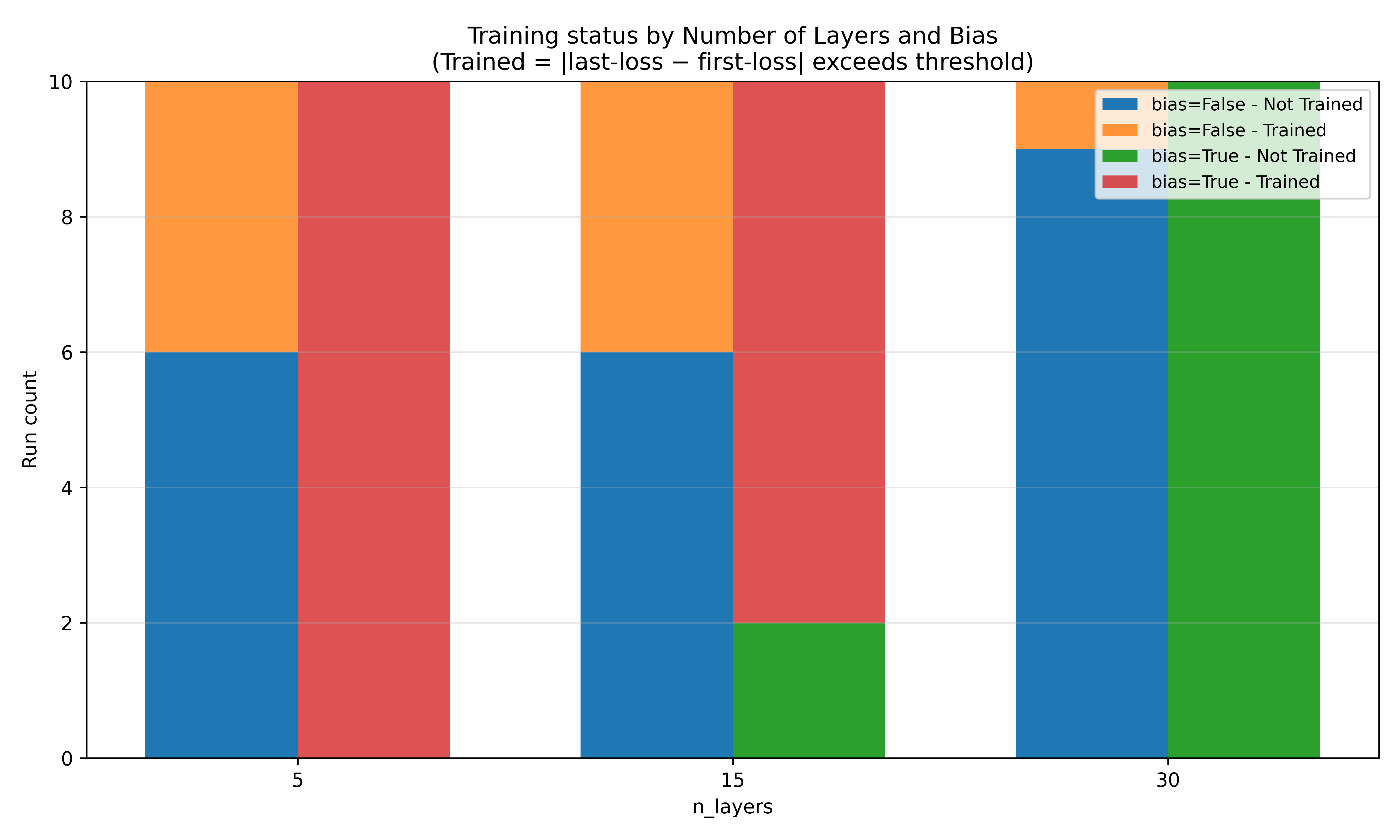}
    \caption{Aggregated statistics of successfully training on the Covertype dataset}
    \label{fig:exp_succ_training}
\end{figure*}

It can be noted in Figure \ref{fig:exp_succ_training} that initializing the bias term did improve the initialization. As there is a higher success rate of actually training when the bias is set, as illustrated in Figure \ref{fig:exp_succ_training}, when the bias is set; however, even with the bias term, the increased number of layers causes the network to stop training thoroughly as demonstrated for $n=30$, where only a single sample of the network trained. Even for the smaller layer numbers, only $n=5$ with the bias correctly set trained at a 100\% success rate. This is clearly not a good sign, as the Covertype in itself is a simple dataset that networks have no issues training on \citet{Olson2018Data-drivenProblems}. We compare against an equivalent feedforward network to demonstrate the expected training loss of these models, as illustrated by the green models in Figure \ref{fig:exp_group_loss}. For the feedforward network, the bias is set to the default PyTorch initialization scheme. 

\section{Conclusion}

This article has demonstrated that the variance of feedforward layers decays at a superlinear rate, which causes issues when using deep 1-Lipschitz feedforward networks. The problem is that the network's output and gradient variances decay to zero, which halts training of the deep network. 
\par
The decay issue was noted when assuming a weight initialization from the generalized normal distributions (Normal, Uniform, and Laplace); as such, initializing with the standard Kaiming methodology causes a problem. While a solution to forward propagation was demonstrated by setting the bias term to an appropriate level, the vanishing backward propagation variance remains to be addressed and will be addressed in future work. 
\par
In addition, the work of \citet{Araujo2023} demonstrates the architecture of a 1-Lipschitz network implemented with a residual network structure; however, due to the more complicated interdependence between components, the layer variance for this type of general structure will also be the focus of future work.

\subsubsection*{Declaration of Generative AI and AI-assisted technologies in the writing process}

During the preparation of this work, the authors used ChatGPT (OpenAI) in order to assist with improving the clarity and fluency of the English text. After using this tool, the author reviewed and edited the content as needed and take full responsibility for the content of the publication.

\subsection*{Funding sources}

This research was supported in part by the U.S. Army Corps of Engineers Engineering Research and Development Center, Construction Engineering Research Laboratory under Grant W9132T23C0013.

\subsection*{Declaration of competing interest}

The authors declare that they have no known competing financial interests or personal relationships that could have appeared to influence
the work reported in this paper.



\printcredits

\bibliographystyle{cas-model2-names}

\bibliography{main}





\end{document}